\documentclass[twoside,11pt]{article}

\usepackage[colorlinks = false]{hyperref}   

\usepackage{booktabs}       %
\usepackage{amsfonts}       %
\usepackage{nicefrac}       %
\usepackage{microtype}      %
\usepackage{xcolor}  
\usepackage{amsmath}
\usepackage{amssymb}
\usepackage{graphicx}
\usepackage{subcaption}
\usepackage{relsize}
\usepackage{mathrsfs}
\usepackage{multicol}
\usepackage{wrapfig}
\usepackage{selectp}
\usepackage{comment}
\hypersetup{ hidelinks }

\usepackage[nohyperref]{jmlr2e}

\usepackage{vwcol}

\newcommand{\R}{{\mathbb{R}}}

\newcommand{\esp}{{\mathbb{E}}}

\newcommand{\Prob}{\mathbb{P}}
\newcommand{\Esp}{\mathbb{E}}

\definecolor{ForestGreen}{cmyk}{0.91,0,0.88,0.12}
\colorlet{pierrem}{ForestGreen}

\DeclareMathOperator*{\ReLU}{ReLU}
\DeclareMathOperator*{\Var}{Var}

\ShortHeadings{Scaling ResNets in the Large-depth Regime}{Marion et al.}
\firstpageno{1}

\usepackage{lastpage}
\jmlrheading{26}{2025}{1-\pageref{LastPage}}{6/22; Revised
2/25}{2/25}{22-0664}{Pierre Marion, Adeline Fermanian, Gérard Biau, and Jean-Philippe Vert}
\ShortHeadings{Scaling ResNets in the Large-depth Regime}{Marion, Fermanian, Biau, and Vert}

\begin{document}

\title{Scaling ResNets in the Large-depth Regime}
\author{\name Pierre Marion\thanks{Now at EPFL} 
\email pierre.marion@mines.org \\
\addr Sorbonne Universit\'{e}, CNRS, Laboratoire de Probabilit\'{e}s, Statistique et Mod\'{e}lisation \\ 
F-75005 Paris, France
\AND 
\name Adeline Fermanian \thanks{Now at Califrais}
\email adeline.fermanian@califrais.fr \\
\addr MINES ParisTech, PSL Research University, CBIO, F-75006 Paris, France \\
Institut Curie, PSL Research University, F-75005 Paris, France \\
INSERM, U900, F-75005 Paris, France
\AND G\'{e}rard Biau 
\email gerard.biau@sorbonne-universite.fr \\
\addr Sorbonne Universit\'{e}, CNRS, Laboratoire de Probabilit\'{e}s, Statistique et Mod\'{e}lisation \\
Institut universitaire de France\\
F-75005 Paris, France
\AND Jean-Philippe Vert\thanks{Now at Owkin}
\email jean-philippe.vert@mines.org \\
\addr Google Research, Brain team, Paris, France
}

\editor{Joan Bruna}

\maketitle

\begin{abstract}%
Deep ResNets are recognized for achieving state-of-the-art results in complex machine learning tasks. However, the remarkable performance of these architectures relies on a training procedure that needs to be carefully crafted to avoid vanishing or exploding gradients, particularly as the depth $L$ increases. No consensus has been reached on how to mitigate this issue, although a widely discussed strategy consists in scaling the output of each layer by a factor $\alpha_L$. We show in a probabilistic setting that with standard i.i.d.~initializations, the only non-trivial dynamics is for $\alpha_L = \nicefrac{1}{\sqrt{L}}$---other choices lead either to explosion or to identity mapping. This scaling factor corresponds in the continuous-time limit to a neural stochastic differential equation, contrarily to a widespread interpretation that deep ResNets are discretizations of neural ordinary differential equations. By contrast, in the latter regime, stability is obtained with specific correlated initializations and $\alpha_L = \nicefrac{1}{L}$. Our analysis suggests a strong interplay between scaling and regularity of the weights as a function of the layer index. Finally, in a series of experiments, we exhibit a continuous range of regimes driven by these two parameters, which jointly impact performance before and after training.
\end{abstract}


\section{Introduction}
\label{sec:intro}

We begin by introducing the general context on deep residual networks, before stating our contributions and discussing related work.

\subsection{Deep Residual Neural Networks}
Residual neural networks (ResNets), introduced by \citet{he2016deep} in the field of computer vision, were the first deep neural network models successfully trained with several thousand layers. 
Since then, extensive empirical evidence has shown that increasing the depth leads to significant improvements in performance, while raising new challenges in terms of training \citep[e.g.,][]{wang2022deepnet}. From a high-level perspective, the key feature of ResNets is the presence of skip connections between successive layers. In mathematical terms, this means that the $(k+1)$-th hidden state $h_{k+1} \in \R^d$ follows sequentially from the previous hidden state via the recurrence relation 
\begin{equation}\label{eq:abstract_resnet}
h_{k+1} = h_k + f(h_k, \theta_{k+1}), \quad 0 \leqslant k \leqslant L-1,
\end{equation}
where $f(\cdot, \theta_{k+1}):\R^d \to \R^d$ is the layer function parameterized by $\theta_{k+1} \in \R^p$ and $L$ is the number of layers. The skip connection corresponds to the addition of $h_k$ on the right-hand side of~\eqref{eq:abstract_resnet}, which is absent in classical feedforward networks.
This refinement prevents instability issues during training when $L$ is large, provided training is performed in a careful way \citep{He2015DelvingDI}. The idea of adding skip connections has become common practice in the field of deep learning, and is today incorporated in many other models such as Transformers in natural language processing \citep{vaswaniTransformer}. 
For simplicity, in the rest of the paper, we continue to use the terminology ResNets to denote any architecture of the form \eqref{eq:abstract_resnet}, keeping in mind that this framework goes beyond the original model of \citet{he2016deep}.
 
The most common architectures have 50-150 layers, but ResNets can be trained with depths up to the order of thousand layers \citep{he2016identity}. Yet, the training procedure needs to be carefully crafted to avoid vanishing or exploding gradients, particularly as the depth increases. As pointed out by, e.g., \citet{shao2020IsNormalizationIndispensable},
these instabilities are related to a shift in the magnitude of the variance of a signal as it passes through the network. In the original approach of \citet{he2016deep}, the issue was mitigated by adding a normalization step, called batch normalization \citep{ioffe2015batchnorm}, which rescales the output of each layer via centering and unit variance normalization. However, this normalization stage introduces practical and theoretical difficulties, among which computational overhead and strong dependence on the batch size \citep[see][and the references therein]{brock2021characterizing}. A widespread alternative to stabilize training in deep models, explored for example by \citet{yang2017mean}, \citet{arpit2019initialize}, \citet{zhang2019convergence}, and \citet{de2020batchNormBiasesTowardIdentity}, is to incorporate a scaling factor $\alpha_L$ in front of the residual term in \eqref{eq:abstract_resnet}, yielding the model
\begin{equation}
h_{k+1} = h_k + \alpha_L f(h_k, \theta_{k+1}), \quad 0 \leqslant k \leqslant L-1.  \label{eq:scaled-discrete-resnet}
\end{equation}
There is strong evidence that this scaling factor $\alpha_L$ should depend on $L$, without however any consensus to date on the exact form of this dependence, nor on the mathematical grounding of the approach. Thus, despite progresses on the empirical side, the mathematical forces in action behind the stability of deep ResNets are still poorly understood, although they are key to unlock training at arbitrary depth. 

Our goal in the present paper is to take a step forward towards a better theoretical understanding of deep ResNets by providing a thorough probabilistic analysis of the sequence $(h_k)_{0 \leqslant k \leqslant L}$ at initialization when $L$ is large, and by leveraging a continuous-time interpretation of model \eqref{eq:scaled-discrete-resnet} via the so-called neural ordinary differential equation (neural ODE, \citealp{chen2018neural}) paradigm. In a nutshell, our results highlight the intimate connection that exists at initialization between stability of the learning process, the regularity of the weights, and the scaling factor $\alpha_L$. 
We offer in particular a proper mathematical grounding on why and how to choose the parameter $\alpha_L$ as a function of the depth $L$ and the distribution of the weights.

\subsection{Our Contributions}

\textit{Scaling at initialization.}
The optimal parameters of ResNets are learned by minimizing some empirical risk function via a gradient descent algorithm.
As highlighted for example by \citet{yang2017mean}, \citet{hanin2018howToStartTraining}, and  \citet{arpit2019initialize}, a good parameter initialization of this learning phase plays a major role
in the quality of the learned model,  in particular to avoid vanishing gradients and deadlock at initialization, or exploding gradients and quick divergence of the model parameters at the beginning of training. Moreover, a good initialization allows the use of larger learning rates, which have been shown to correlate with better generalization \citep{jastrzkebski2017three}.
It is thus of great interest to study and understand the role played by scaling of deep ResNets at initialization. This is the context in which we place ourselves in the sequel.

At initialization stage, the weights $(\theta_k)_{1 \leqslant k \leqslant L}$ are usually chosen as (realizations of) independent and identically distributed (i.i.d.) random variables, which typically follow a uniform or Gaussian distribution on $\R^p$. Accordingly, the sequence $(h_k)_{0 \leqslant k \leqslant L}$ that results from the recursion \eqref{eq:scaled-discrete-resnet} for a given input to the network takes the form of a sequence of random variables that are not~i.i.d. but are actually a martingale. Thus, denoting informally by $\mathscr{L}$ the differentiable loss associated with the learning task (classification or regression), the distributions of $(h_k)_{0 \leqslant k \leqslant L}$ and $(\frac{\partial \mathscr{L}}{\partial h_k})_{0 \leqslant k \leqslant L}$ as $L$ becomes large carry useful information on the stability of training. 
For instance, exploding gradients in the backpropagation phase of learning correspond to the fact that, with high probability, $\|\frac{\partial \mathscr{L}}{\partial h_0}\| \gg \|\frac{\partial \mathscr{L}}{\partial h_L}\|$, where $\|\cdot\|$ denotes the Euclidean norm.
Our first contribution, in Section \ref{sec:scaling-iid}, is to provide thorough mathematical statements on the behavior of these distributions (both for finite and infinite $L$), depending on the value of $\alpha_L$. Among other results, we show that only the choice $\alpha_L \approx \nicefrac{1}{\sqrt{L}}$ yields a non-trivial behavior at initialization, thereby confirming empirical findings in the literature \citep{arpit2019initialize,de2020batchNormBiasesTowardIdentity}. For $\alpha_L \gg \nicefrac{1}{\sqrt{L}}$, the norms explode exponentially fast with $L$, which is inappropriate for training. For $\alpha_L \ll \nicefrac{1}{\sqrt{L}}$, the network is almost equivalent to identity, that is, $h_L \approx h_0$. The analysis of the different cases as a function of~$\alpha_L$ is mathematically involved and makes extensive use of concentration tools from random matrix theory.

\noindent \textit{The continuous approach.} 
As noticed by several authors \citep{chen2018neural,weinan2019mean,thorpe2018deep}, model \eqref{eq:scaled-discrete-resnet} with a scaling factor $\alpha_L=\nicefrac{1}{L}$ (and not $\nicefrac{1}{\sqrt{L}}$) is formally similar to the discretization of a differential equation. Thus, when $L$ tends to infinity, the weights and hidden states change continuously with the layer according to the equation
\begin{equation} \label{eq:abstract-continuous-resnet}
    \frac{dH_t}{dt} = f(H_t, \Theta_t), \quad t \in [0,1].
\end{equation}
Here, time $t$ is the continuous analogue of the layer index $k$, $H:[0,1] \to \R^d$ is a continuous-time hidden state, and $\Theta:[0,1] \to \R^p$ a continuous-time parameter.
This important connection between ResNets and differential equations has been identified in the past years under the umbrella name of neural ODE.
Since the original article of \citet{chen2018neural}, this point of view has led to the development of a variety of new continuous-time models, together with innovative architectures and efficient training algorithms \citep{chang2018antisymmetricrnn,grathwohl2018ffjord,kidger2021efficient}. The neural ODE paradigm also enabled to leverage the rich theory of differential equations to better understand the mechanisms at work behind deep ResNets \citep{weinan2019mean,fermanian2021framing}. However, there is a debated question in the neural ODE community about the choice $\alpha_L=\nicefrac{1}{L}$, which guarantees convergence of the discrete model \eqref{eq:scaled-discrete-resnet} to its continuous-time counterpart \eqref{eq:abstract-continuous-resnet}. As a matter of fact, it seems that this choice is guided by more mathematical than practical considerations, and several authors have suggested that it is inconsistent with what is done in practice \citep{cohen2021scaling,bayer2022resnetsRoughPath}. Moreover, letting $\alpha_L=\nicefrac{1}{L}$ is somewhat contradictory with the results discussed above, which highlighted that the only non-trivial limit at initialization is $\alpha_L=\nicefrac{1}{\sqrt{L}}$.
Thus, as a second contribution, we clarify the problem  in Section \ref{sec:discrete-to-continuous} by leveraging
our previous results on stability. We show that the value $\alpha_L=\nicefrac{1}{\sqrt{L}}$ corresponds in the continuous world to a neural stochastic differential equation (SDE) of the form \eqref{eq:abstract-continuous-resnet},
where now $\Theta:[0,1] \to \R^p$ takes the form of a continuous-time stochastic process, typically a Brownian motion. By contrast, we also prove that the neural ODE regime with $\alpha_L=\nicefrac{1}{L}$ corresponds to the limit of a ResNet, not with i.i.d.~weights as considered before, but with more complex and correlated weight distributions. For these weight distributions, the scaling $\alpha_L=\nicefrac{1}{L}$ is also a critical value between explosion and identity.

Going further, our third contribution is to exhibit in Section \ref{sec:experiments} a continuous range of regimes that are controlled by the choice of $\alpha_L$ (beyond the cases $\nicefrac{1}{\sqrt{L}}$ and $\nicefrac{1}{L}$)  and the distribution of $(\theta_k)_{1 \leqslant k \leqslant L}$ at initialization, derived from a continuous-time process $\Theta$ with a regularity different from a Brownian motion. More precisely, we show experimentally that there is a strong interplay (with the same three cases---explosion, identity mapping, non-trivial behavior) between the choice of $\alpha_L$ and the regularity of $(\theta_k)_{1 \leqslant k \leqslant L}$ as a function of the layer index $k$. In addition, empirical evidence suggests that this interplay impacts both the behavior and performance of the networks during training, beyond initialization.

\subsection{Related Work}

The choice of scaling for ResNets has been discussed in many papers, without however reaching a clear consensus on the form this scaling factor should take. 
For instance, \citet{hanin2018howToStartTraining} state that stability requires $\alpha_L \leqslant \nicefrac{1}{L}$, while \citet{zhang2019convergence} show that $\alpha_L \leqslant \nicefrac{1}{\sqrt{L}}$ is enough to ensure stability. On the other hand, \citet{cohen2021scaling} claim that the scaling factor observed in practice in trained ResNets is of the form $\nicefrac{1}{L^\beta}$ with $\beta \approx 0.7$. Other authors have proposed more complex choices for $\alpha_L$ \citep[e.g.,][]{zhang2019fixup,shao2020IsNormalizationIndispensable}. Taking another point of view, \citet{de2020batchNormBiasesTowardIdentity} observe that batch normalization is empirically equivalent to taking a $\nicefrac{1}{\sqrt{L}}$ normalization factor. \citet{bachlechner2021rezero} suggest learning a scaling parameter $\alpha_k$ that is allowed to vary from one layer to another, whereas, in \eqref{eq:discrete-resnet}, $\alpha_L$ is kept constant across layers. These authors observe a great acceleration for training compared to traditional ResNets with no scaling. They also suggest a similar architecture for Transformers and then notice that $\alpha_k \approx \nicefrac{1}{L}$ at the end of training.

Closest to our analysis at initialization are the papers of \citet{arpit2019initialize} and \citet{zhang2019convergence}. \citet{arpit2019initialize} develop a theoretical analysis based on mean field approximation that suggests that a scaling factor $\alpha_L = \nicefrac{1}{\sqrt{L}}$ prevents vanishing/exploding gradients at initialization, and provide experimental evidence that this approach is competitive with batch normalization. However, the authors do not provide rigorous mathematical statements for the three different cases $\alpha_L \ll \nicefrac{1}{\sqrt{L}}$, $\alpha_L \approx \nicefrac{1}{\sqrt{L}}$, and $\alpha_L \gg \nicefrac{1}{\sqrt{L}}$, nor do they highlight the connection with the continuous-time interpretation. Interestingly, the idea of exploiting the martingale structure to analyze the magnitude of the hidden states is present in \citet{zhang2019convergence}, who study the convergence of
gradient descent for over-parameterized ResNets with different values of $\alpha_L$. Nevertheless, they consider a specific model with Gaussian weights, and only provide asymptotic results when both width and depth tend to infinity. The asymptotic limit in this particular regime has been studied by, e.g., \citet{allenzhu2019convergence} and \citet{hayou2021stable}. We depart from this point of view by considering a finite-width setting.

The connection between the choice of scaling and the continuous-time point of view has previously been noticed by \citet{zhang2019robustResNet}, then studied in detail by \citet{cohen2021scaling} and \citet{cont2022asymptotic}. These authors show that, under assumptions on the form of the weights, it is possible to derive limiting (stochastic or ordinary) differential equations for the hidden states. However, they do not discuss the transition between these two regimes, nor do they link differential equations regimes with the stability of the network.

\section{Scaling at Initialization} \label{sec:scaling-iid}
Our goal in this section is to study the effect of the scaling factor $\alpha_L$ on the stability of ResNets at initialization, assuming that the weights are i.i.d.~random variables. We start by making more precise the model and the learning problem introduced in \eqref{eq:abstract_resnet}.

\subsection{Model and Assumptions}  \label{subsec:model}

\textit{Model.}
The data is a sample of $n$ 
pairs $(x_i, y_i)_{1 \leqslant i \leqslant n}$, where $x_i$ is the input vector in $\R^{n_{\textnormal{in}}}$ and $y_i \in \R^{n_{\textnormal{out}}}$ is the output vector to be predicted. This setting includes regression and classification (after one-hot encoding of the labels).
Specifying the informal recurrence~\eqref{eq:abstract_resnet}, for any input $x \in \R^{n_\textnormal{in}}$, we consider the output $F_\pi(x) \in \R^{n_\textnormal{out}}$ of the ResNet model defined by
\begin{align}
\begin{split}
    h_0 &= A x, \\
  h_{k+1} &= h_k + \alpha_L V_{k+1} g(h_k, w_{k+1}), \quad 0 \leqslant k \leqslant L-1, \label{eq:discrete-resnet} \\
    F_\pi(x) &= B h_L,
\end{split}
\end{align}
where $\alpha_L > 0$ is the scaling factor of the ResNet and $\pi = (A, B, (w_k)_{1 \leqslant k \leqslant L}, (V_k)_{1 \leqslant k \leqslant L})$ are its parameters, with $A \in \R^{d \times n_{\textnormal{in}}}$, $B \in \R^{n_{\textnormal{out}} \times d}$, $w_{k} \in \R^p$ and $V_{k} \in \R^{d\times d}$ for $k=1,\ldots,L$. The almost-everywhere differentiable function $g: \R^d \times \R^p \to \R^d$ encodes the choice of architecture.
We note that the model includes initial and final linear layers in order to map the input space  $\R^{n_\textnormal{in}}$ into the space of hidden states $\R^d$, and symmetrically to map the last hidden state $h_L$ into the output space $\R^{n_{\textnormal{out}}}$. These two transformations are of little interest to us, since we mostly focus on the behavior of the sequence of hidden states $(h_k)_{0 \leqslant k \leqslant L}$. Let us finally notice that the results of this section can be adapted to hidden layers that do not have the same width, at the cost of increased technicality.

An important feature of model \eqref{eq:discrete-resnet} is that the layer function takes the form of a matrix-vector multiplication, which will prove crucial to make use of concentration results on random matrices. We stress that this setting is standard in practice and that it encompasses many  types of ResNets. It includes for example simple ResNets where $g(h, w) = \sigma(h)$ with $\sigma$ the activation function, and the original ResNets from \citet{he2016deep}, which have
\begin{equation*}
    g(h, w) = \textnormal{ReLU}(W h + b),
\end{equation*}
where the parameter is a pair $w=(W,b)$ with $W \in \R^{d \times d}$ a weight matrix and $b \in \R^d$ a bias, and ReLU: $x \mapsto \max(x,0)$ is applied element-wise. 
This setting also includes attention layers, where $g$ corresponds to the scaled dot-product between keys and queries, as well as convolutional layers. %
The assumptions made below do not cover these latter cases, making it an interesting avenue for future research to adapt our approach to these cases.

Throughout the article, we let $\ell: \R^{n_{\textnormal{out}}} \times \R^{n_{\textnormal{out}}} \rightarrow \R_+$ be a  loss function, differentiable w.r.t.~its first parameter, for example the squared loss or the cross-entropy loss. The objective of learning is to find the optimal parameter $\pi$ that minimizes the empirical risk $\mathscr{L}(\pi) = \sum_{i=1}^n \ell(F_{\pi}(x_i), y_i)$.

\medskip \noindent \textit{Probabilistic setting at initialization.} The minimization of the empirical risk is usually performed by stochastic gradient descent or one of its variants \citep[][Chapter 8]{goodfellow2016deep}.
The gradient descent is initialized by choosing the weights as (realizations of) i.i.d.~random variables. The parameters $w_1, V_1, \dots, w_L, V_L$ in model \eqref{eq:discrete-resnet} are therefore assumed to be an i.i.d.~collection of random variables, where we recall that $w_k \in \R^p$ and $V_{k} \in \R^{d \times d}$  parameterize the $k$-th layer of the network. In this stochastic context, the successive hidden states $h_0, \dots, h_L$ given a fixed input $x$
are also random variables, but their distribution is not i.i.d.---in fact, under our assumptions, this sequence is a martingale. To avoid unnecessary technicalities, we assume that the  sequence $(h_k)_{0 \leqslant k \leqslant L}$ is non-atomic. This is for example the case if the distribution of the parameters is absolutely continuous w.r.t.~the Lebesgue measure. In particular, this ensures that the sequence $(h_k)_{0 \leqslant k \leqslant L}$ almost surely does not hit the non-differentiability points of $g$.

It is stressed that the distribution of the parameters are assumed to be independent of the depth, so that all the dependence on $L$ is captured in the scaling factor $\alpha_L$. This model enables us to consider multiple architectures at once, via the function $g$. By contrast, some authors formulate the problem of scaling as a choice of the variance at initialization \citep[e.g.,][]{yang2017mean,wang2022deepnet}, which makes the analysis architecture-dependent. However, for a given architecture, these two approaches are essentially equivalent since $\Var(\alpha_L V_k) = \alpha_L^2 \Var(V_k)$. 

The quantity $\|h_L-h_0\| /\|h_0\|$ carries key information on the behavior of the network at initialization. On the one hand, if $\|h_L-h_0\| \ll \|h_0\|$, the network is essentially equal to the identity function. On the other hand, if $\|h_L-h_0\| \gg \|h_0\|$, the output of the network explodes. 
An intermediate situation is when $\|h_L-h_0\| \approx \|h_0\|$. In addition, another source of information is provided by the gradients of the hidden states with respect to the empirical risk~$\mathscr{L}$. If  $\|\frac{\partial \mathscr{L}}{\partial h_0} - \frac{\partial \mathscr{L}}{\partial h_L}\| \ll \|\frac{\partial \mathscr{L}}{\partial h_L}\|$, the gradients do not change as they flow through the network, which means that the exact same information is backpropagated throughout the network. Conversely, if $\|\frac{\partial \mathscr{L}}{\partial h_0} - \frac{\partial \mathscr{L}}{\partial h_L}\| \gg \|\frac{\partial \mathscr{L}}{\partial h_L}\|$, the gradients explode during backpropagation.
By exploiting the martingale structure of $(\|h_k\|)_{0 \leqslant k \leqslant L}$, as well as state-of-the-art concentration inequalities for random matrices with sub-Gaussian entries, we provide in this section probabilistic bounds on the magnitude of these various quantities.

\medskip \noindent \textit{Assumptions.} Some assumptions are needed on the choices of architecture and initialization. Recall that a centered real-valued random variable $X$ is said to be $s^2$ sub-Gaussian \citep[Chapter 3]{vanHandel2016} if for all $\lambda \in \R$, $\Esp(\exp(\lambda X)) \leqslant\exp(\nicefrac{\lambda^2s^2}{2})$. 
The sub-Gaussian property is a constraint on the tail of the probability distribution. As an example, Gaussian random variables on the real line are sub-Gaussian and so are bounded random variables. 

The following assumptions will be needed throughout the section: for any $1 \leqslant k \leqslant L$,
\begin{itemize}
    \item[$(A_1)$] For some $s \geqslant 1$, the entries of $\sqrt{d}V_{k}$ are centered i.i.d.~$s^2$ sub-Gaussian random variables, independent of $d$ and $L$, with unit variance.
    \item[$(A_2)$] For some $C > 0$, independent of $d$ and $L$, and for any $h \in \R^d$,
    \[ \frac{\|h\|^2}{2} \leqslant\Esp\big(\|g(h, w_{k})\|^2\big) \leqslant\|h\|^2 \quad \text{and} \quad \Esp\big(\|g(h, w_{k})\|^8\big) \leqslant C \|h\|^8.\]
\end{itemize}

Assumption $(A_1)$ is mild and satisfied by all initializations used in practice. For example, the classical Glorot initialization \citep{glorot2010training}---which is the default implementation in the Keras package \citep{chollet2015keras}---takes the entries of $V_{k}$ as uniform $\mathcal{U}(-\sqrt{3/d}, \sqrt{3/d})$ variables. This means that $\sqrt{d}V_{k}$ is initialized with $\mathcal{U}(-\sqrt{3}, \sqrt{3})$ random variables, which satisfy $(A_1)$.
Other examples include the Gaussian $\mathcal{N}(0, 1/d)$ initialization of \citet{He2015DelvingDI} and, for example, initialization with Rademacher variables.

The first part of Assumption $(A_2)$ ensures that $g( \cdot, w_{k})$ is not too far away from being an isometry in expectation. The second part is more technical and, roughly, allows to upper bound the deviations of the norm of $g(h_{k-1}, w_{k})$. Our next Proposition \ref{prop:standard-resnet-verifies-assumptions} shows that most classical ResNet architectures verify Assumption $(A_2)$. For the sake of readability, these models, together with their parameters, are summarized in Table \ref{tab:examples} below.
\begin{table}[ht]
    \centering
    \begin{tabular}{llll}
    \toprule
    {\bf} & {\bf Name} & {\bf Recurrence relation} & {\bf Parameters of $g$} \\
    \midrule
    \texttt{res-1}  & Simple ResNet & $h_{k+1} = h_k + \alpha_L V_{k+1} \sigma (h_k)$ & $w_{k+1} = \emptyset$ \\
    \texttt{res-2}  & Parametric ResNet & $h_{k+1} = h_k + \alpha_L V_{k+1} \sigma (W_{k+1}h_k)$ & $w_{k+1} = W_{k+1}$ \\
    \texttt{res-3}  & Original ResNet & $h_{k+1} = h_k + \alpha_L V_{k+1} \ReLU (W_{k+1} h_k)$ & $w_{k+1} = W_{k+1}$ \\
      \bottomrule
    \end{tabular}
    \caption{Examples of ResNet architectures considered in the paper. In the first two cases, the activation function $\sigma$ is such that, for all $x \in \R, a |x| \leqslant|\sigma(x)| \leqslant b |x|$, $\nicefrac{1}{\sqrt{2}} \leqslant a < b \leqslant1$. In the last two cases, $W_{k+1} \in \R^{d \times d}$.}
    \label{tab:examples}
\end{table}

\begin{proposition} \label{prop:standard-resnet-verifies-assumptions}
Let \texttt{res-1}, \texttt{res-2}, and \texttt{res-3} be the models defined in Table \ref{tab:examples}. Then
\begin{itemize}
    \item[$(i)$] Assumption $(A_2)$ is satisfied for \texttt{res-1}.
    \item[$(ii)$] Assumption $(A_2)$ is satisfied for \texttt{res-2} and \texttt{res-3}, as soon as the entries of $\sqrt{d}W_{k+1},$ $0 \leqslant k \leqslant L-1,$ are centered i.i.d.~sub-Gaussian random variables, independent of $d$ and $L$, with unit variance.
\end{itemize}
\end{proposition}

In the models \texttt{res-1} and \texttt{res-2}, $\sigma$ can be, for instance, taken as the parametric ReLU function, i.e., $\sigma(x) = x_+ + s x_-$, where $x_+$ (resp. $x_-$) denotes the positive (resp. negative) part and the slope $s \in [\nicefrac{1}{\sqrt{2}}, 1]$ is a parameter of the model.
Parametric ReLU, also known as leaky ReLU \citep{maas2013rectifier}, has been shown to outperform standard ReLU for image datasets \citep{xu2015empirical}.
Observe also that \texttt{res-2} differs from \texttt{res-3} since the classical ReLU function is defined by $\ReLU(x) = x_+$ and thus does not satisfy the condition $|\sigma(x)| \geqslant a|x|$.  
Note that there is no bias term in these three models, as this term is commonly initialized to zero, and we are interested in the behavior at initialization.

\begin{remark}
An important architecture that is not covered by our setting is $h_{k+1} = h_k + \alpha_L \sigma(W_{k+1}h_k)$, where compared to \texttt{res-1} the weight matrix is located inside the activation function $\sigma$. This is due to the fact that the residual branch writes as a matrix-vector product in our model \eqref{eq:discrete-resnet}, which enables us to leverage matrix concentration inequalities (see, e.g., Lemmas \ref{lemma:bound-deviations-linear} and \ref{lemma:bound-deviations-quadratic}). However, we check numerically that qualitatively similar results are obtained in this case (see Appendix~\ref{apx:experimental-setting}). This leads us to believe that similar concentration results should hold true when the matrix-vector product is followed by an element-wise activation function. We leave this mathematical study for future work.
\end{remark}

\subsection{Probabilistic Bounds on the Norm of the Hidden States} \label{subsec:prob-bounds-hidden-states}

The next two propositions describe how the quantity $\|h_L - h_0\|/\|h_0\|$ changes as a function of $L \alpha_L^2$. Proposition \ref{prop:init-forward-high-prob} provides a high-probability bound of interest when $L \alpha_L^2 \ll 1$. In this case, we see that, with high probability, the network acts as the identity function, directly mapping $h_0$ to $h_L$.
On the other hand, Proposition \ref{prop:init-forward-high-prob2} provides information in the two cases $L \alpha_L^2 \gg 1$ and $L \alpha_L^2 \approx 1$. When $L \alpha_L^2 \gg 1$, the lower bound $(i)$ indicates an explosion with high probability of the norm of the last hidden state. On the other hand, when $L \alpha_L^2 \approx 1$, the bounds $(i)$ and $(ii)$ show that $h_L$ randomly varies around $h_0$ with fluctuation sizes bounded from below and above.

\begin{proposition} \label{prop:init-forward-high-prob}
Consider a ResNet \eqref{eq:discrete-resnet} such that Assumptions $(A_1)$ and $(A_2)$ are satisfied. If $L \alpha_L^2 \leqslant1$, then, for any $\delta \in (0, 1)$, with probability at least $1-\delta$,
$$\frac{\|h_L - h_0\|^2}{\|h_0\|^2} \leqslant\frac{2 L \alpha_L^2}{\delta}.$$
\end{proposition}

\begin{proposition} \label{prop:init-forward-high-prob2}
Consider a ResNet \eqref{eq:discrete-resnet} such that Assumptions $(A_1)$ and $(A_2)$ are satisfied. 

\begin{itemize}
    \item[$(i)$] Assume that $d \geqslant 64$ and $\alpha_L^2 \leqslant\frac{2}{(\sqrt{C}s^4 + 4 \sqrt{C} + 16 s^4)d}$. Then, for any $\delta \in (0, 1)$, with probability at least $1-\delta$,
\[
\frac{\|h_L - h_0\|^2}{\|h_0\|^2} > \exp \left(\frac{3L\alpha_L^2}{8} - \sqrt{\frac{11L\alpha_L^2}{d\delta}}\right) - 1,
\]
provided that
\begin{equation}    \label{eq:condition-prop-init-forward}
2L \exp\left(-\frac{d}{64\alpha_L^2s^2}\right) \leqslant\frac{\delta}{11}.    
\end{equation}
    \item[$(ii)$] Assume that $\alpha_L^2 \leqslant \frac{1}{\sqrt{C}(d + 128 s^4)}$. Then, for any $\delta \in (0, 1)$, with probability at least $1-\delta$,
\[
\frac{\|h_L-h_0\|^2}{\|h_0\|^2} < \exp \left(L\alpha_L^2 + \sqrt{\frac{5L\alpha_L^2}{d\delta}} \right) + 1.
\]
\end{itemize}

\end{proposition}

We recall that the constants $s$ and $C$ appearing in  Proposition~\ref{prop:init-forward-high-prob2} are defined respectively by Assumptions $(A_1)$ and $(A_2)$.
Moreover, note that the assumptions of Proposition \ref{prop:init-forward-high-prob2} on $d$ and $\alpha_L$ are mild, since in the learning tasks where deep ResNets are involved, one typically has $\alpha_L=\nicefrac{1}{L^\beta}$ with $\beta > 0$, $d \geqslant 10^2$ and $L \geqslant 10^2$.
Furthermore, the assumption on $\alpha_L$ is a technical assumption used to simplify the final expression. It is also possible to derive a proof without this assumption, at the cost of a more intricate result.
Note also that condition \eqref{eq:condition-prop-init-forward} is not severe since, when $d$ and $L$ are large, it encompasses all reasonable values of $\delta$.
Propositions \ref{prop:init-forward-high-prob} and~\ref{prop:init-forward-high-prob2} are interesting in the sense that they provide finite-depth high-probability bounds on the behavior of the hidden states, depending on the magnitude of~$L\alpha_L^2$. The results become clearer by letting $\alpha_L = \nicefrac{1}{L^\beta}$, with $\beta > 0$, as shown in the following corollary.

\begin{corollary}   \label{corollary:forward}
Consider a ResNet \eqref{eq:discrete-resnet} such that Assumptions $(A_1)$ and $(A_2)$ are satisfied, and let $\alpha_L = \nicefrac{1}{L^\beta}$, with $\beta > 0$.
\begin{itemize}
    \item[$(i)$] If $\beta > \nicefrac{1}{2}$, then
\[
\frac{\|h_L - h_0\|}{\|h_0\|} \xrightarrow[L\rightarrow\infty]{\Prob} 0.
\]
    \item[$(ii)$] If $\beta < \nicefrac{1}{2}$ and $d \geqslant 9$, then
\[
\frac{\|h_L - h_0\|}{\|h_0\|} \xrightarrow[L\rightarrow\infty]{\Prob} \infty.
\]
    \item[$(iii)$] If $\beta = \nicefrac{1}{2}$, $d \geqslant 64$, $L \geqslant (\frac{1}{2}\sqrt{C}s^4 + 2 \sqrt{C} + 8 s^4)d + 96 \sqrt{C}s^4$, then, for any $\delta \in (0, 1)$, with probability at least $1-\delta$,
\[
\exp \left(\frac{3}{8} - \sqrt{\frac{22}{d\delta}}\right) - 1 < \frac{\|h_L - h_0\|^2}{\|h_0\|^2} < \exp \left(1 + \sqrt{\frac{10}{d\delta}} \right) + 1,
\]  
provided that
\[
2L \exp \Big( -\frac{Ld}{64s^2} \Big) \leqslant \frac{\delta}{11}.
\]
\end{itemize}
\end{corollary}

\begin{figure}
    \centering
    \includegraphics[width=\textwidth]{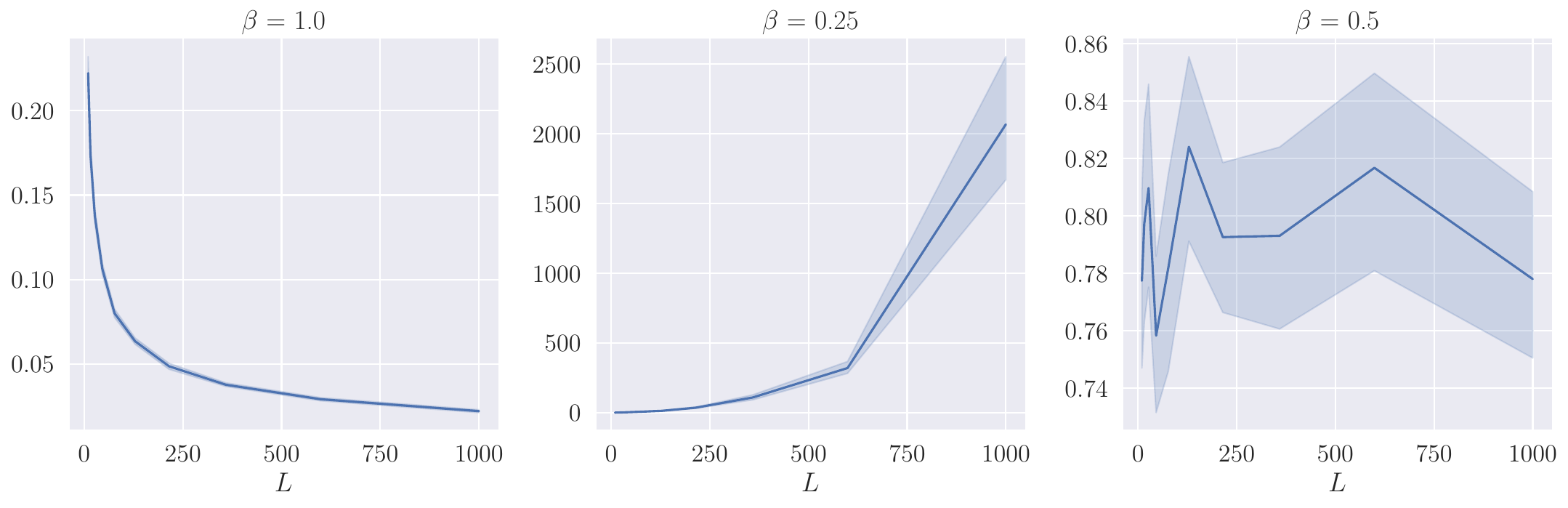}
    \caption{Evolution of $\|h_L - h_0\|/\|h_0\|$ as a function of $L$ for different values of~$\beta$ and an i.i.d.~$\mathcal{U}(-\sqrt{3/d}, \sqrt{3/d})$ initialization of model \texttt{res-3}, with $d=40$. The input is a random Gaussian observation $x$ in dimension $n_{\textnormal{in}} = 64$. The experiment is repeated with $50$ independent randomizations.}
    \label{fig:scaling_init_norm}
\end{figure}

Corollary \ref{corollary:forward} highlights three different asymptotic behaviors for $\|h_L\|$, depending on the values of $\beta$. 
For $\beta > \nicefrac{1}{2}$, statement $(i)$ tells that $h_L$ converges towards $h_0$ in probability, as $L$ tends to infinity, which means that the neural network is essentially equivalent to an identity mapping. 
On the other hand, for $\beta < \nicefrac{1}{2}$, the norm of $h_L$ explodes with high probability.
Finally, for the critical value $\beta = \nicefrac{1}{2}$, we see that $h_L$ fluctuates around $h_0$, with a fluctuation size independent of $L$. Observe that the lower bound in $(iii)$ is not trivial as soon as $\exp(\nicefrac{3}{8} - \sqrt{\nicefrac{11}{d\delta}}) > 1$, i.e., $d > \nicefrac{99}{64\delta}$. 
The message of Corollary \ref{corollary:forward} is that the only scaling leading to a non-degenerate distribution {at initialization} is for $\beta = \nicefrac{1}{2}$. 

\begin{figure}
    \centering
    \begin{subfigure}[b]{0.49\textwidth}    
        \centering
        \includegraphics[width=\textwidth]{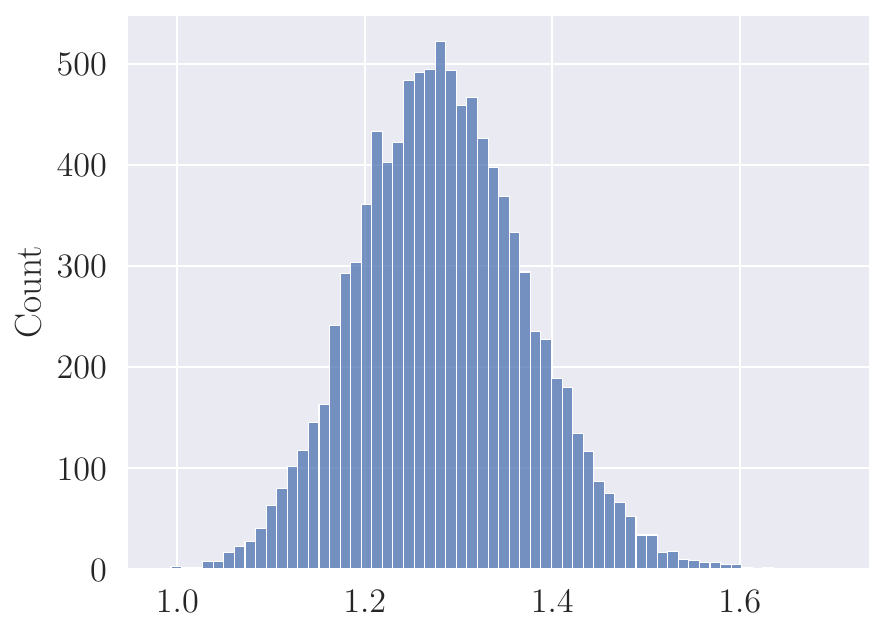}
        \caption{Distribution of $\|h_L\|/\|h_0\|$}
        \label{fig:distribution-forward}
    \end{subfigure}
    \hfill
    \begin{subfigure}[b]{0.49\textwidth}
        \centering
        \includegraphics[width=\textwidth]{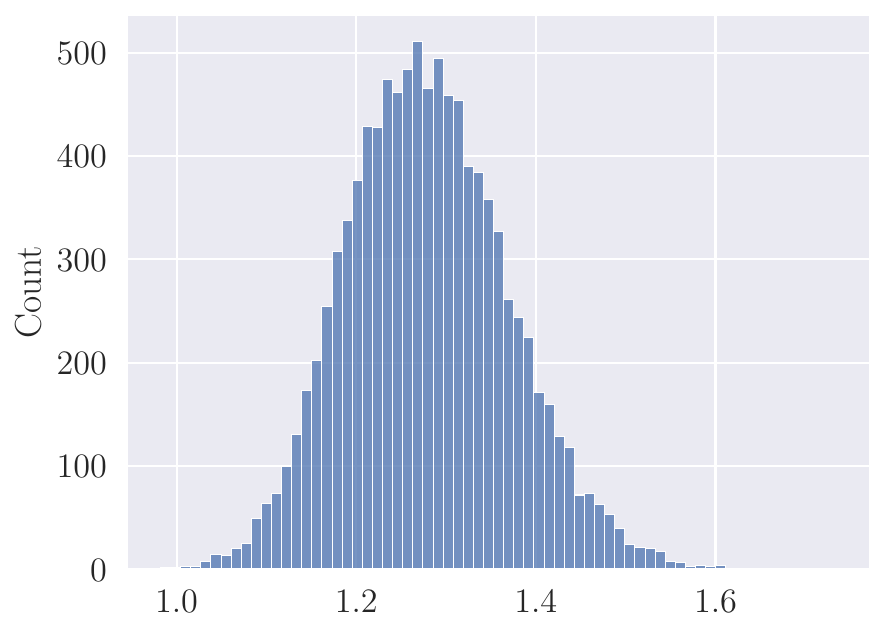}
        \caption{Distribution of $\|\frac{\partial \mathscr{L}}{\partial h_0}\|/ \|\frac{\partial \mathscr{L}}{\partial h_L}\|$}
        \label{fig:distribution-gradients}
    \end{subfigure}
       \caption{Empirical distributions of the norms for $\beta = \nicefrac{1}{2}$, $L=10^3 $, $d=100$. The experiment is repeated with $10^4$ independent randomizations.}
       \label{fig:distributions}
\end{figure}

The three statements of Corollary~\ref{corollary:forward} are illustrated in Figure \ref{fig:scaling_init_norm}. In this experiment, we consider model \texttt{res-3}, a random Gaussian observation $x$ in dimension $n_{\textnormal{in}} = 64$, and parameters initialized with a uniform distribution $\mathcal{U}(-\sqrt{3/d}, \sqrt{3/d})$. We refer to Appendix~\ref{apx:experimental-setting} for a detailed setup of all the experiments of the paper.
Figure \ref{fig:distribution-forward} shows the empirical distribution of $\|h_L\|/\|h_0\|$ when $\beta=\nicefrac{1}{2}$ for a large number of realizations. This figure illustrates in particular that our bounds are reasonably sharp, since the bounds indicate that the first quartile of the distribution is larger than $0.87$ (whereas the first quartile of the empirical histogram is equal to $1.21$) and the third quartile is less than $2.06$ (whereas the third quartile of the empirical histogram is equal to $1.34$). Determining the exact distribution of $\|h_L\|/\|h_0\|$ is an interesting avenue for future research that is beyond the scope of the present article. There is however a strong indication that the ratio follows a log-normal distribution, as confirmed by a normality test on (the log of) the empirical distribution. 

In a nutshell, the proofs of Propositions \ref{prop:init-forward-high-prob} and \ref{prop:init-forward-high-prob2} rest upon controlling of the norm of the hidden states, which obeys the recurrence 
\begin{equation}    \label{eq:rec-norm-hidden-states}
\|h_{k+1}\|^2 = \|h_k\|^2 + \alpha_L^2 \|V_{k+1} g(h_k, w_{k+1})\|^2 + 2 \alpha_L \langle h_k, V_{k+1} g(h_k, w_{k+1}) \rangle, 
\end{equation}
where $\langle \cdot, \cdot \rangle$ denotes the standard scalar product in $\R^d$.
Taking the expectations on both side, one deduces with Assumptions $(A_1)$ and $(A_2)$ that 
\begin{align}    \label{eq:sketch-proof-forward-1}
\Esp \big(\|V_{k+1} g(h_k, w_{k+1})\|^2 \big) &= \Esp\Big(\Esp \big(\|V_{k+1} g(h_k, w_{k+1})\|^2 \big) \, \big| \, h_k, w_{k+1}\Big) \nonumber \\ 
&= \Esp \big(\|g(h_k, w_{k+1})\|^2 \big) \approx \|h_k\|^2    
\end{align}
and 
\begin{equation}    \label{eq:sketch-proof-forward-2}
\Esp \big(\langle h_k, V_{k+1} g(h_k, w_{k+1}) \rangle \big) = \Esp\Big(\Esp\big(\langle h_k, V_{k+1} g(h_k, w_{k+1}) \rangle \big| h_k, w_{k+1}\big)\Big) = 0.
\end{equation}
The equalities \eqref{eq:sketch-proof-forward-1} and \eqref{eq:sketch-proof-forward-2} allow deriving without further work bounds in expectation on $\|h_L\|$, as already observed in an informal manner by \citet{arpit2019initialize}. However, the results we are after are stronger since they involve precise quantification of the fluctuations induced by the initialization through high-probability bounds. A finer control of the deviations of $\|V_{k+1} g(h_k, w_{k+1})\|^2$ and $\langle h_k, V_{k+1} g(h_k, w_{k+1}) \rangle$ is then needed. This involves concentration inequalities on random matrices with sub-Gaussian entries, and martingale arguments. This probabilistic derivation allows improvements over earlier works in the literature that only show stability for $\beta \geq 1$ \citep{hanin2018howToStartTraining,allenzhu2019convergence}. A similar proof technique was already used in \citet{zhang2019convergence} to show in an asymptotic setting explosion of the forward pass for $\beta < \nicefrac{1}{2}$ and stability for $\beta \geq \nicefrac{1}{2}$, in accordance with our results. We extend their result in several ways, most notably by considering a sub-Gaussian initialization relaxing their Gaussian assumption, a more general architecture, and obtaining fully non-asymptotic bounds, while their approach is asymptotic both in width and depth. This makes the mathematical analysis significantly more challenging.  We also introduce a novel distinction between the stability case $\beta=\nicefrac{1}{2}$ and the identity case $\beta > \nicefrac{1}{2}$, showing the presence of non-vanishing fluctuations only for $\beta=\nicefrac{1}{2}$.
For completeness, we note that a revised version of their paper does provide non-asymptotic bounds \citep{zhang2022stabilize}, albeit still in the more restrictive setting of a Gaussian initialization, a specific architecture, and without separating the cases $\beta=\nicefrac{1}{2}$ and $\beta > \nicefrac{1}{2}$.

We next show precise non-asymptotic bounds for the gradients, using the martingale structure of the forward differentiation recurrence, a novel proof technique.

\subsection{Probabilistic Bounds on the Gradients} 

Propositions \ref{prop:init-forward-high-prob} and \ref{prop:init-forward-high-prob2} provide insights on the output of the network when $L$ is large. However, they do not carry information on the backwards dynamics of propagation of the gradients of the loss $p_k = \frac{\partial \mathscr{L}}{\partial h_k} \in \R^d$. Assessing the dynamics of the $(p_k)_{0 \leqslant k \leqslant L}$ as a function of $L$ is important since the behavior of this sequence impacts trainability of the network at initialization. Thus, in this subsection, we are interested in quantifying the magnitude of $\|p_0 - p_L\| / \|p_L\|$, when $L$ is large. Notice that, contrarily to the previous subsection where we were mostly interested in the last hidden state $h_L$, the quantity of interest is now $p_0$ (not $p_L$), the gradient at index $0$. The reason is that the sequence $(p_k)_{0 \leqslant k \leqslant L}$ is defined backwardly, as we will see below. We also stress that $(p_k)_{0 \leqslant k \leqslant L}$ is the sequence of derivatives of the loss w.r.t.~the hidden states $h_k$, and not w.r.t.~the parameters. 
The reason for considering this sequence is that the $p_k$ are involved in the backpropagation algorithm and are therefore essential for assessing the stability of the gradient descent \citep[see, e.g.,][]{arpit2019initialize}.

Analyzing the behavior of the sequence $(p_k)_{0 \leqslant k \leqslant L}$ is challenging since, according to the backpropagation (or reverse-mode differentiation) formula, one has
\begin{equation*} 
p_k = p_{k+1} + \alpha_L \frac{\partial g(h_k, w_{k+1})^\top}{\partial h} V_{k+1}^\top p_{k+1}.
\end{equation*}
Taking the norm,
\[
\|p_k\|^2 = \|p_{k+1}\|^2 + \alpha_L^2 \Big\| \frac{\partial g(h_k, w_{k+1})^\top}{\partial h} V_{k+1}^\top p_{k+1} \Big\|^2 + 2 \alpha_L \Big\langle p_{k+1}, \frac{\partial g(h_k, w_{k+1})^\top}{\partial h} V_{k+1}^\top p_{k+1} \Big\rangle.
\]
Although the equation looks qualitatively similar to \eqref{eq:rec-norm-hidden-states}, it has the unpleasant feature that $p_{k+1}$ depends on the whole sequence of weights $w_1, V_1, \dots,$ $w_{L}, V_{L}$. This forbids applying directly the same proof techniques as for the hidden states due to the lack of adaptation of the $p_k$ to the filtration of the hidden state process.
Simplifying assumptions have sometimes been made to analyze this recurrence equation, for instance assuming independence of $\frac{\partial g(h_k, w_{k+1})}{\partial h}$ and $h_k$ \citep[see, e.g.,][]{yang2017mean}. 
However, such assumption remains a strong requirement, which is not verified for many network architectures (for example model \texttt{res-1}). 
Other authors resort to an $\varepsilon$-net argument \citep{allenzhu2019convergence,zhang2019convergence}.
We tackle the problem from a different point of view and propose an alternative approach based on forward-mode differentiation, valid under a much weaker assumption. The cost we pay is that we obtain results in expectation and not in high probability.

Let us sketch our approach before stating the results more formally.
We denote by $z \in \R^d$ an independent random variable that will be used to assess the magnitude of the gradients. For any $0 \leqslant i,j \leqslant L$, let $\frac{\partial h_j}{\partial h_i} \in \R^{d \times d}$ be the Jacobian matrix of $h_j$ with respect to $h_i$. Recall that the $(m,n)$-th entry of this matrix equals the derivative of the $m$-th coordinate of $h_j$ w.r.t. the $n$-th coordinate of $h_i$. Then, 
letting $q_k(z) = \frac{\partial h_k}{\partial h_0} z$, we have, by the chain rule,
\begin{equation}    \label{eq:general-rel-q_k}
q_{k+1}(z) = \frac{\partial h_{k+1}}{\partial h_{k}} q_k(z) = q_k(z) + \alpha_L V_{k+1} \frac{\partial g(h_k, w_{k+1})}{\partial h} q_{k}(z).
\end{equation}
Identity \eqref{eq:general-rel-q_k}, which is similar to \eqref{eq:discrete-resnet}, expresses $q_{k+1}(z)$ as a function of $q_k(z)$, and therefore respects the flow of information. 
Next, assuming that $z$ is random with a Gaussian distribution, it is possible to express one of our quantities of interest, $\|p_0\| / \|p_L\|$, as a function of the last vector $q_L(z)$. Indeed,
\begin{equation}    \label{eq:forward-diff-main}
\frac{\|p_0\|^2}{\|p_L\|^2}  = \frac{1}{\|p_L\|^2} \Esp_{z \sim \mathcal{N}(0, I_d)} \Big(\big|p_0^\top z\big|^2\Big) = \Esp_{z \sim \mathcal{N}(0, I_d)} \bigg(\Big|\Big(\frac{p_L}{\|p_L\|}\Big)^\top q_L(z)\Big|^2\bigg), 
\end{equation}
where $I_d$ is the identity matrix in $\R^d$ and the second equality is a consequence of 
\[
p_0^\top z = \Big( \frac{\partial \mathscr{L}}{\partial h_0} \Big)^\top z = \Big( \frac{\partial \mathscr{L}}{\partial h_L} \Big)^\top \frac{\partial h_L}{\partial h_0} z = p_L^\top q_L(z).
\]
In summary, the recurrence \eqref{eq:general-rel-q_k} allows us to derive bounds on the norm of $q_L(z)$, which can then transfer to $\|p_0\| / \|p_L\|$ via \eqref{eq:forward-diff-main}. For this, it is necessary to make the following assumption on the ratio $p_L / \|p_L\|$:
\begin{itemize}
    \item[$(A_3)$] Let $b = p_L / \|p_L\|$. Then $\Esp(b | h_L) = 0$ and $\Esp(b^\top b | h_L) = I_d / d$.
\end{itemize}

It is a mild assumption, which is verified for instance if $n_{\textnormal{out}} = 1$ with squared error (for regression) or cross-entropy (for binary classification). In these cases, $p_L / \|p_L\| = \nicefrac{B^\top}{\|B\|_F}$, where $\|\cdot\|_F$ is the Frobenius norm and $B$ is the weight matrix of the last layer. We finally need the following assumption, which is the equivalent of Assumption $(A_2)$ for the gradients.
\begin{itemize}
    \item[$(A_4$)] One has, almost surely,
    \[
    \frac{\|q_k\|^2}{2} \leqslant \Esp \bigg( \Big\| \frac{\partial g(h_k, w_{k+1})}{\partial h} q_{k} \Big\|^2  \Big| h_k, q_{k} \bigg) \leqslant\|q_k\|^2.
    \]
\end{itemize}

Assumption $(A_4)$ is satisfied by all the standard architectures listed in Table \ref{tab:examples}, as shown by the next proposition.

\begin{proposition} \label{prop:standard-resnet-verifies-assumptions-gradients}
Let \texttt{res-1}, \texttt{res-2}, and \texttt{res-3} be the models defined in Table \ref{tab:examples}. Assume that $(A_1)$ is satisfied and $\sigma$ is almost everywhere differentiable, with $a \leqslant \sigma' \leqslant b$. Then
\begin{itemize}
    \item[$(i)$] Assumption $(A_4)$ is satisfied for \texttt{res-1}.
    \item[$(ii)$] Assumption $(A_4)$ is satisfied for \texttt{res-2} and \texttt{res-3}, when the entries of $\sqrt{d}W_{k}, 1 \leqslant k \leqslant L,$ are centered i.i.d. random variables, independent of $d$ and $L$, with unit variance.
\end{itemize}
\end{proposition}

The next two propositions are the counterparts of Proposition \ref{prop:init-forward-high-prob} and Proposition \ref{prop:init-forward-high-prob2} for the gradient dynamics.
\begin{proposition}  \label{prop:gradients}
Consider a ResNet \eqref{eq:discrete-resnet} such that Assumptions $(A_1)$-$(A_4)$ are satisfied. If $L \alpha_L^2 \leqslant 1$, then, for any $\delta \in (0, 1)$, with probability at least $1-\delta$,
$$\frac{\|p_0 - p_L\|^2}{\|p_L\|^2} \leqslant\frac{2 L \alpha_L^2}{\delta}.$$
\end{proposition}

\begin{proposition}  \label{prop:gradients2}
Consider a ResNet \eqref{eq:discrete-resnet} such that Assumptions $(A_1)$-$(A_4)$ are satisfied. Then
\[
\big(1 + \frac{1}{2} \alpha_L^2\big)^L - 1 \leqslant \Esp \bigg( \frac{\|p_0-p_L\|^2}{\|p_L\|^2} \bigg) \leqslant (1 + \alpha_L^2)^L - 1.
\]
\end{proposition}

A simple corollary of the propositions above is as follows.
\begin{corollary}
\label{cor:gradients2}
Consider a ResNet \eqref{eq:discrete-resnet} such that Assumptions $(A_1)$-$(A_4)$ are satisfied, and take $\alpha_L = \nicefrac{1}{L^\beta}$, with $\beta > 0$. Then
\begin{itemize}
    \item[$(i)$] If $\beta > \nicefrac{1}{2}$,
\[
\frac{\|p_0 - p_L\|}{\|p_L\|}  \xrightarrow[L\rightarrow\infty]{\Prob} 0.
\]
    \item[$(ii)$] If $\beta < \nicefrac{1}{2}$,
\[
\Esp \bigg( \frac{\|p_0-p_L\|^2}{\|p_L\|^2} \bigg) \xrightarrow[]{L\rightarrow\infty} \infty.
\]
    \item[$(ii)$] If $\beta = \nicefrac{1}{2}$,
\[
\exp\Big(\frac{1}{2}\Big) - 1 \leqslant \Esp \bigg( \frac{\|p_0-p_L\|^2}{\|p_L\|^2} \bigg) \leqslant \exp(4) - 1.
\]
\end{itemize}
\end{corollary}

Corollary \ref{cor:gradients2} is illustrated in Figure \ref{fig:scaling_init_gradient}. The experimental protocol is the same as in Figure~\ref{fig:scaling_init_norm}, but we now track $p_0$ and $p_L$, the gradients of the loss $\mathscr{L}$ with respect to the first and the last hidden states. In accordance with our results, when $\beta > \nicefrac{1}{2}$, the gradient remains the same from one layer to another (left plot). On the other hand, the middle plot clearly shows that when $\beta < \nicefrac{1}{2}$ the gradient explodes. 
Once again, the case $\beta = \nicefrac{1}{2}$ (right plot) is the only one for which the distribution of gradients at initialization is non-trivial. Figure \ref{fig:distribution-gradients} illustrates that the empirical distribution of gradients in this case also seems to be log-normal.

Our findings extend results from the literature showing explosion of the backward pass for $\beta < \nicefrac{1}{2}$ and stability for $\beta \geq \nicefrac{1}{2}$, both in an asymptotic setting \citep{allenzhu2019convergence,zhang2019convergence,hayou2021stable} and a non-asymptotic setting \citep{zhang2022stabilize}. Similar comparisons can be drawn with these papers as in the previous section (see end of Section \ref{subsec:prob-bounds-hidden-states}). Furthermore, we emphasize again that making use of the forward-mode formulation of the gradients differs from these previous works, which resorted either to an infinite-width setting or to backward-mode differentiation and an $\varepsilon$-net argument.

\begin{figure}
    \centering
    \includegraphics[width=\textwidth]{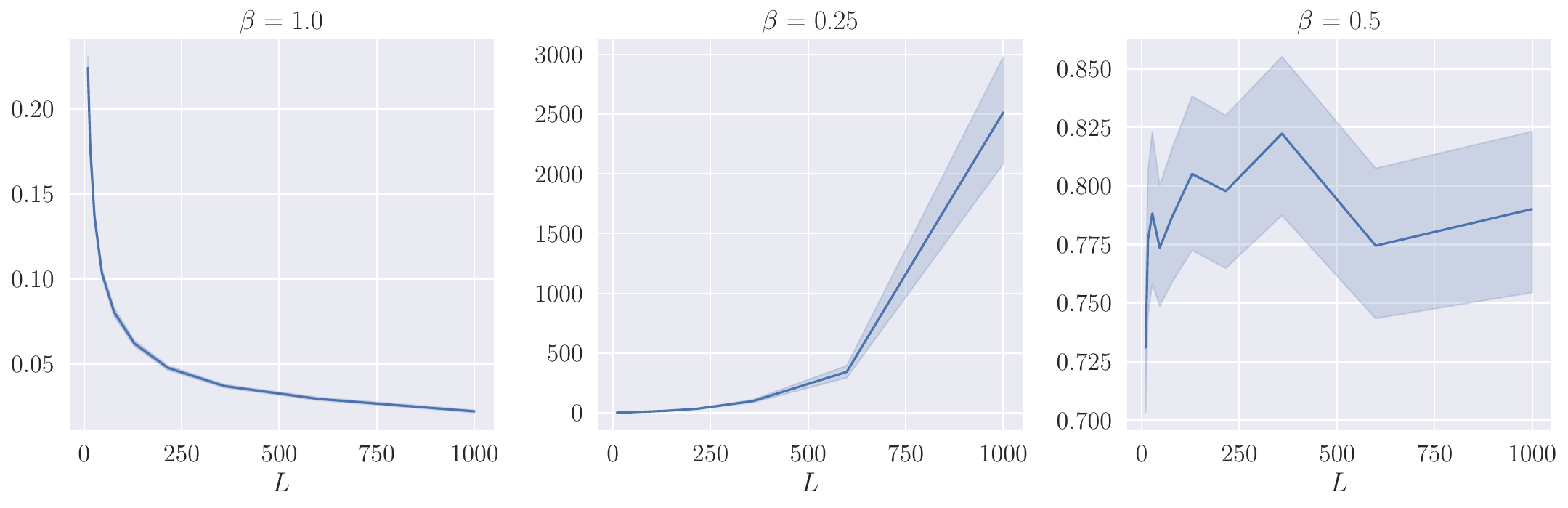}
    \caption{Evolution of $\|p_0-p_L\| / \|p_L\|$ as a function of $L$ for different values of~$\beta$ and an i.i.d.~$\mathcal{U}(-\sqrt{3/d}, \sqrt{3/d})$ initialization of model \texttt{res-3}, with $d=40$. The input is a random Gaussian observation $x$ in dimension $n_{\textnormal{in}} = 64$. The experiment is repeated with $50$ independent randomizations.}
    \label{fig:scaling_init_gradient}
\end{figure}

In summary, this and the previous subsection both point towards the same conclusion: there are three different cases, depending on the value of $\beta$---explosion when $\beta < \nicefrac{1}{2}$, non-degenerate limit when $\beta=\nicefrac{1}{2}$, and identity when $\beta>\nicefrac{1}{2}$. In the explosion case, it is well known that the network cannot be trained \citep{yang2017mean}. The theory thus points out that the value $\nicefrac{1}{2}$ plays a pivotal role. Remarkably, this value has a specific interpretation in the continuous-time point of view of ResNets, in terms of SDE. This is the topic that we address in the next section.

\section{Scaling in the Continuous-time Setting}
\label{sec:discrete-to-continuous}

Starting with the discrete ResNet \eqref{eq:discrete-resnet}, it is tempting to let $L$ go to infinity and consider the network as the discretization of a differential equation where the layer index $k \in \{0, \dots, L\}$ is replaced by the time index $t \in [0,1]$. This interpretation of deep neural networks has been popularized by \citet{chen2018neural} and is often referred to as the neural ODE paradigm.
Notice that this setting is different from the so-called mean-field analysis, where the width of the network is assumed to be infinite beforehand. In our setting, the width $d$ is assumed to be finite and fixed.

\subsection{Convergence Towards a SDE in the Large-depth Regime}\label{subsec:convergence_sde}
One of the main messages of Section \ref{sec:scaling-iid} is that the standard initialization with i.i.d.~parameters leads to a non-degenerate model for large values of $L$ only if $L\alpha_L^2 \approx 1$ (Propositions \ref{prop:init-forward-high-prob} and \ref{prop:init-forward-high-prob2}), or, equivalently, if $\beta=\nicefrac{1}{2}$ when $\alpha_L = \nicefrac{1}{L^\beta}$ (Corollary \ref{corollary:forward}). Remarkably, in the continuous-time limit, this regime corresponds to the discretization of a SDE. Indeed, consider for simplicity the  (discrete) ResNet model \texttt{res-1}
\begin{equation} \label{eq:discrete_sde_resnet}
    h_0 = Ax, \quad h_{k+1} = h_k + \frac{1}{\sqrt{L}} V_{k+1}\sigma(h_k), \quad 0 \leqslant k \leqslant L-1,
\end{equation}
where the entries of all $(V_{k})_{1 \leqslant k \leqslant L}$ are assumed to be i.i.d.~$\mathcal{N}(0, \nicefrac{1}{d})$. Recall the following definition:

\begin{definition}\label{def:brownian-motion}
A one-dimensional Brownian motion $(B_t)_{t \in [0, 1]}$ is a continuous-time stochastic process with $B_0=0$, almost surely continuous, with independent increments, and such that for any $0 \leqslant s < t \leqslant 1$, $B_{t} - B_s \sim \mathcal{N}(0, t-s)$.
\end{definition}

Now, let $\mathbf{B}:[0,1] \to \R^{d \times d}$ be a $(d \times d)$-dimensional Brownian motion, in the sense that the $(B_{ij})_{1 \leqslant i, j \leqslant d}$ are independent one-dimensional Brownian motions.  Thus, for any $0 \leqslant k \leqslant L-1$ and any $1 \leqslant i,j \leqslant d$, we have
\begin{equation*} 
    \mathbf{B}_{\nicefrac{(k+1)}{L}, i, j} - \mathbf{B}_{\nicefrac{k}{L}, i, j} \sim \mathcal{N}\Big(0, \frac{1}{L}\Big),
\end{equation*}
and the increments for different values of $(i,j,k)$ are independent. As a consequence, the recurrence~\eqref{eq:discrete_sde_resnet} is equivalent in distribution to the recurrence 
\begin{equation*}
    h_{k+1}^\top = h_k^\top + \frac{1}{\sqrt{d}} \sigma(h_k^\top)(\mathbf{B}_{\nicefrac{(k+1)}{L}} - \mathbf{B}_{\nicefrac{k}{L}}), \quad 0 \leqslant k \leqslant L-1.
\end{equation*}
(Note that this is true because $V_{k+1}$ has the same distribution as $V_{k+1}^\top$.)
We recognize the Euler-Maruyama discretization \citep{kloeden1992numerical} on the $\{k/L, 0 \leqslant k \leqslant L\}$ mesh of the SDE 
\begin{equation} \label{eq:resnet_sde_init}
    H_0 = Ax, \quad dH_t^\top = \frac{1}{\sqrt{d}} \sigma(H_t^\top) d\mathbf{B}_t, \quad t \in [0, 1],
\end{equation}
where the output of the network is now a function of the final value of $H$, that is, $H_1$. 
The link between the discrete ResNet \eqref{eq:discrete_sde_resnet} and the SDE \eqref{eq:resnet_sde_init} is formalized in the next proposition. 

\begin{proposition} \label{prop:resnet_convergence_sde_init}
    Consider the \texttt{res-1} model, where
    the entries of $V_{k}$ are i.i.d.~Gaussian $\mathcal{N}(0, \nicefrac{1}{d})$ random variables. Assume that the activation function $\sigma$ is Lipschitz continuous. Then the SDE \eqref{eq:resnet_sde_init} has a unique solution $H$ and, for any $0 \leqslant k \leqslant L$,
    \begin{equation*}
        \esp \big( \| H_{\nicefrac{k}{L}} - h_k \| \big) \leqslant\frac{c}{\sqrt{L}},
    \end{equation*}
    for some $c > 0$.
\end{proposition}

Notice that the requirement that $\sigma$ is Lipschitz continuous is satisfied by most classical activation functions, including ReLU.
This proposition is interesting for several reasons. First, the scaling $\beta= \nicefrac{1}{2}$, which is exactly the one that yields a non-trivial dynamics at initialization, corresponds in the continuous world to a remarkably `simple' model of diffusion. This shows that very deep neural networks properly initialized with i.i.d.~weights are equivalent to solutions of SDE. This analogy opens interesting perspectives for training deep networks using automatic differentiation for solutions of neural SDE \citep{li2020scalable}.%

Second, we stress that the emergence of a SDE instead of an ODE carries an important message. Several authors (including, e.g., \citealp{thorpe2018deep}) have shown that, under appropriate assumptions, a deep ResNet converges in the large depth limit to an ODE and not a SDE. The reason why we obtain a SDE here is intrinsically connected with the choice of i.i.d.~initialization for the weights, which makes a Brownian motion appear at the limit, as highlighted above. In other words, the i.i.d.~initialization, the choice $\beta = \nicefrac{1}{2}$ (the relevant critical value exhibited in Section \ref{sec:scaling-iid}), and the emergence of a SDE are intimately linked together. On the other hand, the case $\beta = 1$ matches with an ODE if the initialization is not i.i.d., as we will see in Subsection \ref{subsec:smooth}.

Finally, we point out that Proposition \ref{prop:resnet_convergence_sde_init} states the convergence of a ResNet towards a SDE for the basic architecture \texttt{res-1} and for Gaussian initialization. The extension to more general settings is an interesting direction of research, although clearly beyond the scope of the present paper (see, e.g., \citealp{peluchetti2020infinitely}, and \citealp{cohen2021scaling}, for results in this direction). 

\subsection{Scaling in the Neural ODE Setting} 
\label{subsec:smooth}

\textit{Convergence towards an ODE.}
The basic message of our Proposition \ref{prop:resnet_convergence_sde_init} is that an i.i.d.~initialization, together with $\beta=\nicefrac{1}{2}$, leads to a SDE rather than an ODE. A natural question is then whether a different choice of weight distributions (at initialization) and scaling can lead to a classical neural ODE.

To answer this question and leave the world of i.i.d.~initialization, we assume that the weights $(V_k)_{1 \leqslant k \leqslant L}$ and $(w_k)_{1 \leqslant k \leqslant L}$ are discretizations of smooth functions $\mathscr{V}: [0, 1] \rightarrow \R^{d \times d}$ and $\mathscr{W}: [0, 1] \rightarrow \R^p$. We then consider the general iteration \eqref{eq:discrete-resnet} with $\alpha_L=\nicefrac{1}{L}$, that is,
\begin{align}
\begin{split}
    h_0 &= A x, \quad h_{k+1} = h_k + \frac{1}{L} V_{k+1} g(h_k, w_{k+1}), \quad 0 \leqslant k \leqslant L-1, \label{eq:smooth-regime}
\end{split}
\end{align}
where $V_k = \mathscr{V}_{k/L}$ and $w_k = \mathscr{W}_{k/L}$.
Of course, it is still possible to consider $(V_k)_{1 \leqslant k \leqslant L}$ (resp. $(w_k)_{1 \leqslant k \leqslant L}$) as random variables, by letting $(\mathscr{V}_t)_{t \in [0,1]}$  (resp. $(\mathscr{W}_t)_{t \in [0,1]}$) be a continuous-time stochastic process.  
In this model, we shall need the following assumption: 
\begin{itemize}
    \item[$(A_5)$] For any $1 \leqslant k \leqslant L$, one has $V_k = \mathscr{V}_{k/L}$ and $w_k = \mathscr{W}_{k/L}$, where the stochastic processes $\mathscr{V}$ and $\mathscr{W}$ are almost surely Lipschitz continuous. 
\end{itemize}

More precisely, almost surely, there exist  $K_{\mathscr{V}}, K_{\mathscr{W}}$ %
such that, for any $s,t \in [0,1]$,
\begin{equation*}
\|\mathscr{V}_t - \mathscr{V}_s \| \leqslant K_{\mathscr{V}}|t-s|, \quad \|\mathscr{W}_t - \mathscr{W}_s \| \leqslant K_{\mathscr{W}}|t-s|.
\end{equation*}
A typical model that satisfies Assumption $(A_5)$ is obtained by letting the entries of $\mathscr{V}$ and $\mathscr{W}$ be independent Gaussian processes with expectation zero and squared exponential covariance $K(x, x') = \exp (-\frac{(x-x')^2}{2\ell^2})$, where $\ell > 0$ \citep{lederer2019uniform}. Note that the Lipschitz constants may themselves be random, depending on the Gaussian process sample.

We shall also need the following requirement on $g$, which is satisfied by all our models as soon as $\sigma$ is Lipschitz continuous:
\begin{itemize}
 \item[$(A_6)$] The function $g$ is Lipschitz continuous on compact sets, in the sense that for any compact $\mathscr{P} \subseteq \R^p$, there exists $K_{\mathscr{P}} > 0$ such that, for all $h,h' \in \R^d$, $w \in \mathscr{P}$, 
    $$\|g(h, w) - g(h', w)\| \leqslant K_{\mathscr{P}} \|h-h'\|,$$
 and for any compact $\mathscr{D} \subseteq \R^d$, there exists $K_{\mathscr{D},\mathscr{P}} > 0$ such that, for all $h \in \mathscr{D}$, $w, w' \in \mathscr{P}$, 
    $$\|g(h, w) - g(h, w')\| \leqslant K_{\mathscr{D},\mathscr{P}} \|w-w'\|.$$
\end{itemize}

Under Assumptions $(A_5)$ and $(A_6)$, the recurrence~\eqref{eq:smooth-regime} almost surely converges towards the neural ODE given by
\begin{align}
\begin{split}
    H_0 &= Ax, \quad dH_t = \mathscr{V}_t g(H_t, \mathscr{W}_t) dt, \quad t \in [0, 1], \label{eq:neural-ode}
\end{split}
\end{align}
as shown by the proposition below.
\begin{proposition} \label{prop:neural-ode}
Consider model \eqref{eq:smooth-regime} such that Assumptions $(A_5)$ and $(A_6)$ are satisfied.  Then the ODE \eqref{eq:neural-ode} has a unique solution $H$, and, almost surely, there exists some $c>0$ such that, for any $0 \leqslant k \leqslant L$,
\[
\|H_{\nicefrac{k}{L}} - h_k \| \leqslant\frac{c}{L}.
\]
\end{proposition}

It should be stressed that the transition from the discrete recurrence \eqref{eq:smooth-regime} to the continuous-time differential equation \eqref{eq:neural-ode} relies on the assumptions that the weight sequences $(w_k)_{1 \leqslant k \leqslant L}$ and $(V_k)_{1 \leqslant k \leqslant L}$ are the discretizations of smooth limiting processes $\mathscr{W}$ and $\mathscr{V}$ on the one hand, and that the scaling $\alpha_L$ is chosen as $\nicefrac{1}{L}$ on the other hand. 
From a practical perspective, Proposition \ref{prop:neural-ode} shows that it is possible to initialize ResNets in the ODE regime, by choosing a smooth stochastic process, discretizing it at each layer, and taking a $\nicefrac{1}{L}$ scaling. This is in sharp contrast with the results of Sections \ref{sec:scaling-iid} and \ref{subsec:convergence_sde}, which show that the usual i.i.d.~procedure leads to a neural SDE.

\medskip \noindent \textit{Stability and scaling.} Assuming that the weights of the network are discretizations of a smooth function (Assumption $(A_5)$), it is possible to obtain stability results, depending on the value of $\beta$, similarly to what has been done in Section \ref{sec:scaling-iid}. We show below that $\beta=1$ is a critical value, by examining the hidden states, in the same way as $\beta = \nicefrac{1}{2}$ is a critical value in the i.i.d.~setting. Similar results can be shown for the gradients. We begin by a proposition handling the cases $\beta > 1$ and $\beta = 1$.

\begin{proposition}\label{prop:stability-neural-ode}
Consider a ResNet \eqref{eq:discrete-resnet} such that Assumptions $(A_5)$ and $(A_6)$ are satisfied. Let $\alpha_L = \nicefrac{1}{L^\beta}$, with $\beta  >0$.
\begin{itemize}
    \item[$(i)$] If $\beta > 1$, then, almost surely,
\[
\frac{\|h_L-h_0\|}{\|h_0\|} \xrightarrow[]{L\rightarrow\infty} 0.
\]
    \item[$(ii)$] If $\beta = 1$, then, almost surely, there exists some $c > 0$ such that
    \begin{equation*}
    	\frac{\|h_L-h_0\|}{\|h_0\|} \leqslant c.
    \end{equation*}

\end{itemize}
\end{proposition}

The explosion case ($\beta < 1$) is more delicate to deal with. We prove it for a linear model, and leave for future work the extension to more general cases.

\begin{proposition}\label{prop:explosion-neural-ode}
Consider the $\texttt{res-1}$ model, taking $\sigma$ as the identity function. Assume that Assumption $(A_5)$ is satisfied and that $\mathscr{V}_0^T$ has a positive eigenvalue. Let $\alpha_L = \nicefrac{1}{L^\beta}$, with $\beta \in (0, 1)$. Then, almost surely,
\[
\max_k \frac{\|h_k-h_0\|}{\|h_0\|} \xrightarrow[]{L\rightarrow\infty} \infty.
\]
\end{proposition}

The assumption of the existence of a positive eigenvalue for $\mathscr{V}_0^\top$ is mild. For instance, if the entries of $\mathscr{V}_0$ are i.i.d.~random variables with finite moments of all order, \citet{gotze2021rate} show that such an eigenvalue exists with probability at least $1-\nicefrac{1}{d}$ for $d$ large enough. Essentially, the proof relies on showing divergence of the hidden states along the eigenvector associated to a positive eigenvalue of $\mathscr{V}_0^T$. The extension to models that do not have an identity activation function is delicate. Indeed, while we expect such divergence to generally occur with a non-linear activation function, it is technically delicate to show as one has to rule out cases where the non-linearity of the activation function induces compensations that prevent divergence.

\begin{figure}
    \centering
    \includegraphics[width=\textwidth]{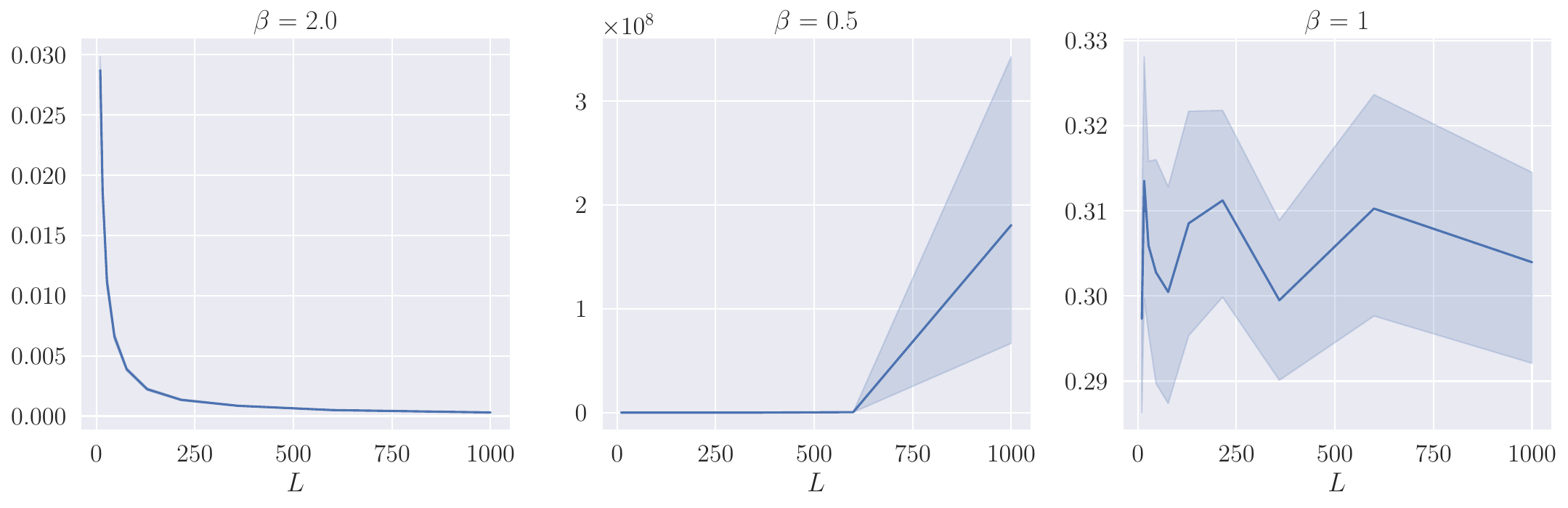}
    \caption{Evolution of $\|h_L-h_0\|/\|h_0\|$ as a function of $L$ for different values of~$\beta$ and a smooth initialization of model \texttt{res-3}, with $d=40$. The input is a random Gaussian observation $x$ in dimension $n_{\textnormal{in}} = 64$. The experiment is repeated with $50$ independent randomizations.}
    \label{fig:smooth_scaling_init_norm}
\end{figure}

\begin{figure}
    \centering
    \includegraphics[width=\textwidth]{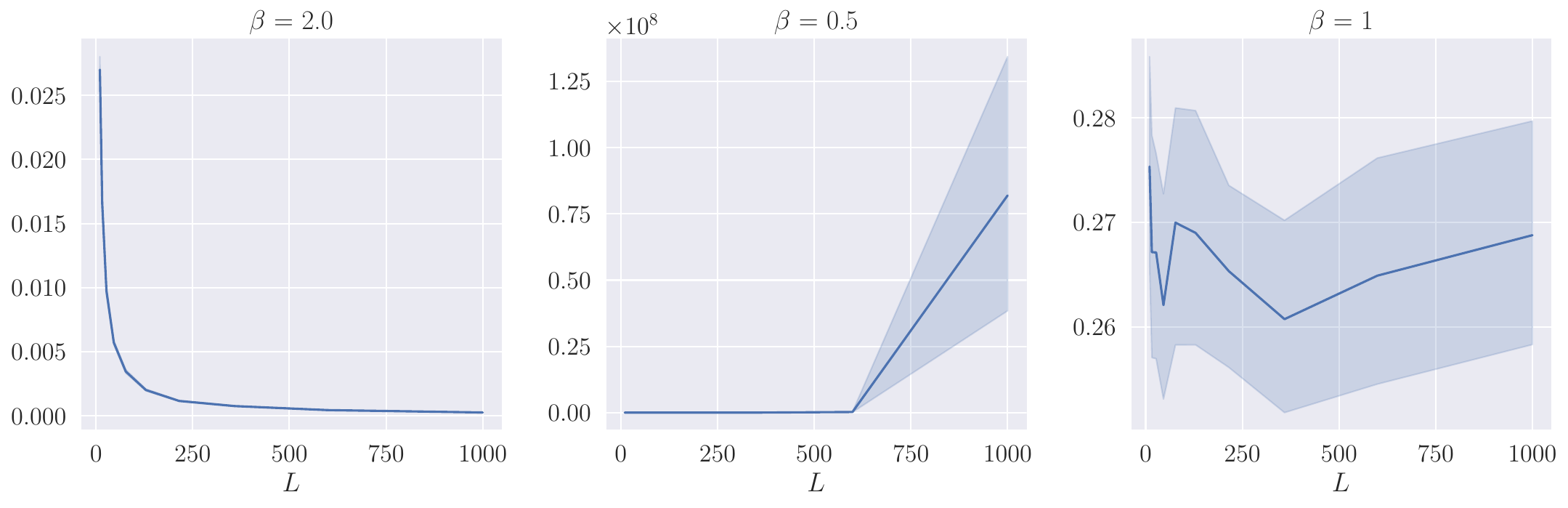}
    \caption{Evolution of $\|p_0-p_L\|/\|p_L\|$ as a function of $L$ for different values of~$\beta$ and a smooth initialization of model \texttt{res-3}, with $d=40$. The input is a random Gaussian observation $x$ in dimension $n_{\textnormal{in}} = 64$. The experiment is repeated with $50$ independent randomizations.}
    \label{fig:smooth_scaling_init_gradient}
\end{figure}

In this setting, we observe experimentally a behavior of the output and of the gradients when $L$ grows large similar to the one explored in Section \ref{sec:scaling-iid}. This is illustrated in Figures \ref{fig:smooth_scaling_init_norm} and \ref{fig:smooth_scaling_init_gradient}, which mirror Figures \ref{fig:scaling_init_norm} and \ref{fig:scaling_init_gradient} in Section \ref{sec:scaling-iid}. The figures clearly show that there exist three cases for the output and for the gradients: an identity case (left plots), an explosion case (middle), and a non-trivial case separating explosion and identity (right). However, the remarkable point is that the separation occurs for $\beta = 1$, and not $\beta=\nicefrac{1}{2}$, as predicted by Propositions \ref{prop:stability-neural-ode} and \ref{prop:explosion-neural-ode}.

\section{Experiments}
\label{sec:experiments}

We experimentally investigate in this section two questions. The first one is to know whether there exists a range of scaling factors $\beta > 0$ and weight initializations, beyond the i.i.d.~and the smooth regimes. The second question is whether our analysis, which pertains to the initialization phase, provides insights into the training phase, beyond initialization.

\subsection{Intermediate Regimes}
\label{subsec:fbm-magnitude}

In order to describe the transition between the i.i.d.~and smooth cases, a possible route is to consider that the weights are increments of a $\gamma$-H\"{o}lder stochastic process. This model is interesting insofar as the Brownian motion (SDE regime) is $(\nicefrac{1}{2}-\varepsilon)$-H\"{o}lder ($\varepsilon>0$) and a Lipschitz continuous stochastic process (ODE regime) is $1$-H\"{o}lder.

In line with the above, in a series of experiments, we initialize the weights as increments of a fractional Brownian motion $(B^H_t)_{t \in [0, 1]}$. Recall that $B^H$ is a continuous-time Gaussian process, starting at zero, with zero expectation for all $t \in [0, 1]$, and covariance function
\[
\Esp(B^{H}_s B^{H}_t) = \frac {1}{2} (|s|^{2H}+|t|^{2H}-|t-s|^{2H}), \quad 0 \leqslant s, t \leqslant 1,
\]
where $H \in (0, 1)$ is called the Hurst index. This index describes the raggedness of the process, with a higher value leading to a smoother process. When $H = \nicefrac{1}{2}$, the process is a standard Brownian motion (Definition \ref{def:brownian-motion}), whose increments are independent by construction. When $H > \nicefrac{1}{2}$, the increments of the process are positively correlated, while if $H < \nicefrac{1}{2}$ the increments are negatively correlated. Importantly, a fractional Brownian motion with Hurst index $H$ is $(H-\varepsilon)$-H\"{o}lder continuous for any $\varepsilon > 0$. In the limit when $H \rightarrow 1$, the trajectories converge to linear functions (whose increments satisfy $(A_5)$).
As an illustration, Figure \ref{fig:fbm} depicts three realizations of a fractional Brownian motion with $H=0.2$ (left), $H=0.5$ (middle), and $H=0.8$ (right).

\begin{figure}
    \centering
    \begin{subfigure}[b]{0.32\textwidth}
        \centering
        \includegraphics[width=\textwidth]{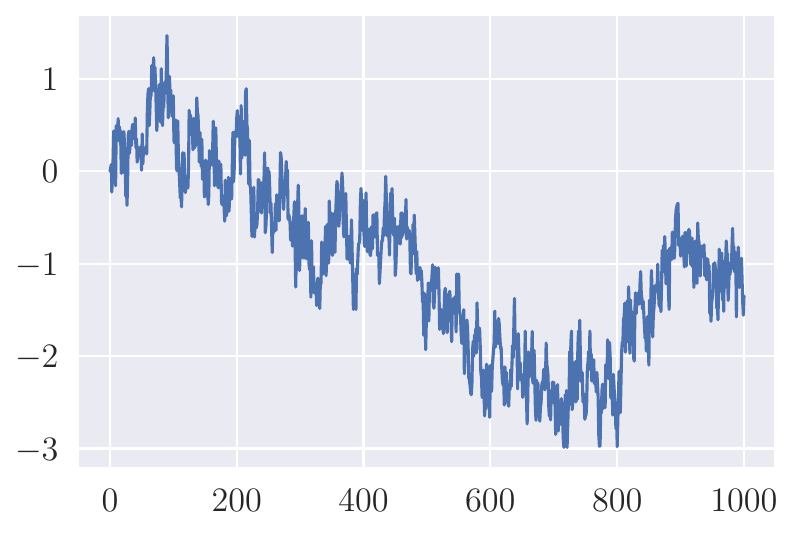}
        \caption{$H=0.2$}
    \end{subfigure}
    \hfill
    \begin{subfigure}[b]{0.32\textwidth}    
        \centering
        \includegraphics[width=\textwidth]{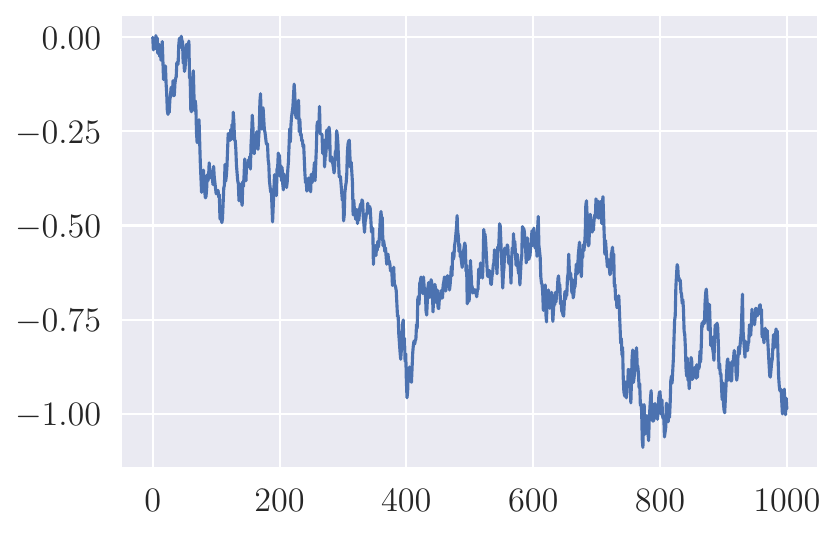}
        \caption{$H=0.5$}
    \end{subfigure}
    \hfill
    \begin{subfigure}[b]{0.32\textwidth}
        \centering
        \includegraphics[width=\textwidth]{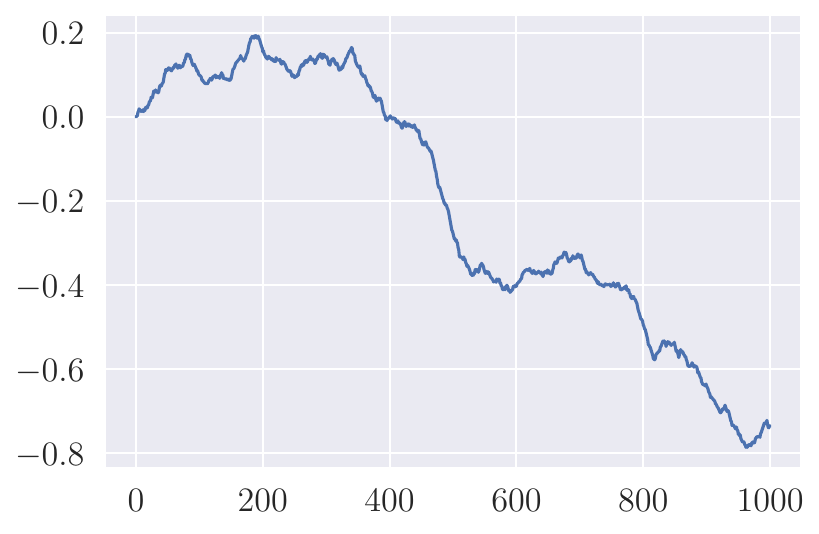}
        \caption{$H=0.8$}
    \end{subfigure}
    \caption{Examples of realizations of a fractional Brownian motion $B^H$ for different Hurst indexes $H$. Note that the smaller the value of $H$, the more irregular the trajectory is.}
    \label{fig:fbm}
\end{figure}

\begin{figure}[ht]
     \centering
     \begin{subfigure}[b]{0.49\textwidth}
         \centering
         \includegraphics[width=\textwidth]{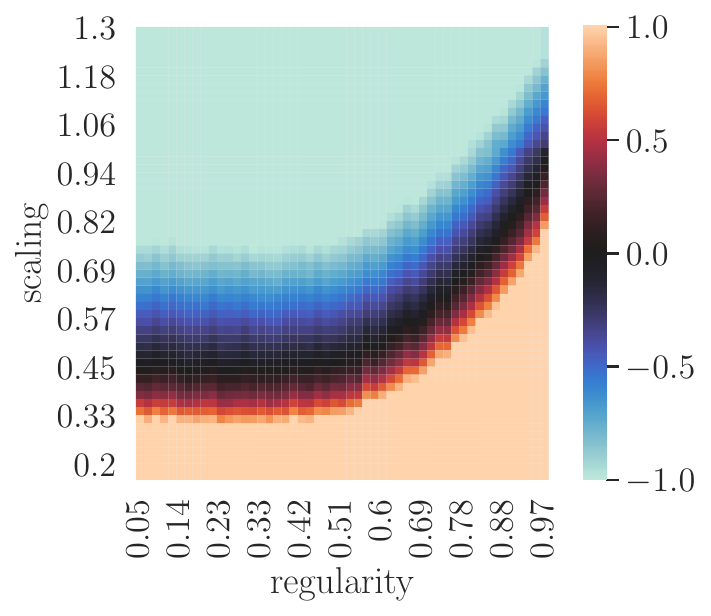}
         \caption{$\ln_{10} (\|h_L-h_0\|/\|h_0\|)$}
     \end{subfigure}
     \hfill
     \begin{subfigure}[b]{0.49\textwidth}
         \centering
         \includegraphics[width=\textwidth]{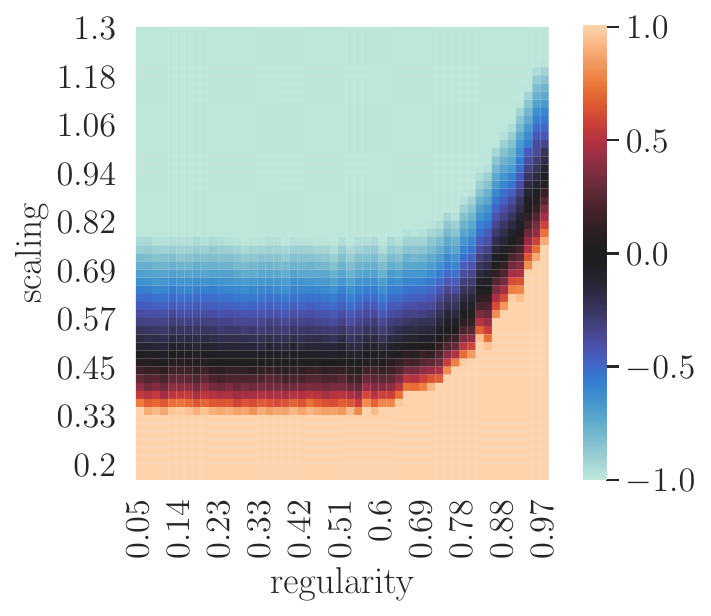}
         \caption{$\ln_{10} (\|p_0-p_L\|/\|p_L\|)$}
     \end{subfigure}
     \caption{Magnitude of the outputs and of the gradients as a function of the regularity of the weights (Hurst index $H$) and of the scaling factor $\beta$. The orange zone corresponds to the explosion case, i.e., $\|h_L-h_0\| \gg \|h_0\|$ and $\|p_0-p_L\| \gg \|p_L\|$. The blue zone corresponds to the identity case, i.e., $\|h_L-h_0\| \ll \|h_0\|$ and $\|p_0-p_L\| \ll \|p_L\|$. Finally, the black zone is an intermediate case, where $\|h_L-h_0\| \approx \|h_0\|$ and $\|p_0-p_L\| \approx \|p_L\|$.}
     \label{fig:heatmap}
\end{figure}
In order to assess the effect of the scaling factor $\beta$ and the Hurst index $H$, we initialize a neural network \texttt{res-3} with $d=40$, $L=1000$, various values of $\beta \in [0.2, 1.3]$, and with weights taken as increments of fractional Brownian motions with various Hurst indices $H \in (0, 1)$.
Figure \ref{fig:heatmap} depicts the empirical magnitude of the output and the gradients at initialization as a function of the Hurst index $H$ and the scaling factor $\beta$. 
First note that we recover the two regimes (i.i.d.~and smooth) discussed so far. For $H=\nicefrac{1}{2}$, the i.i.d.~regime kicks in, with explosion ($\beta<\nicefrac{1}{2}$, orange zone), non-trivial behavior ($\beta=\nicefrac{1}{2}$, black zone), and identity ($\beta>\nicefrac{1}{2}$, blue zone). Likewise, we see at $H=1$ a similar pattern in the smooth regime, with, as predicted by Proposition \ref{prop:stability-neural-ode}, a critical value $\beta=1$. %
Beyond these two specific cases, we observe for an index $H$ varying in $(\nicefrac{1}{2}, 1)$ a whole range of intermediate situations, where the transition between identity and explosion seems to happen for a critical $\beta = H$. Interestingly, for $H < \nicefrac{1}{2}$, the transition seems to saturate at the value $\beta=\nicefrac{1}{2}$.

The take-home message is that the choice of the scaling of a ResNet seems to be closely linked to the regularity of the weights as a function of the layer. More precisely, for all regimes,  the critical scaling factor between explosion and identity seems to have a natural interpretation as the (H\"{o}lder) regularity of the underlying continuous-time stochastic process.
We believe that the mathematical understanding of this connection, beyond the fractional Brownian motion case, is a promising research direction for the future.
Finally, these experiments suggest that it is sensible to initialize a ResNet for any value of the scaling $\beta \in (\nicefrac{1}{2}, 1)$, while avoiding the identity and explosion situations, by simulating a fractional Brownian motion of Hurst index $H=\beta$ and initializing the weights as the increments of this process.

\subsection{Beyond Initialization}

At initialization, before the gradient descent, the distribution of the weights $(w_k)_{1 \leqslant k \leqslant L}$ and $(V_k)_{1 \leqslant k \leqslant L}$ is chosen by the practitioner. By contrast, during and after training, control is lost on these distributions, making the picture more complex. In particular, the existence and characterization of a continuous-time stochastic process whose discretization matches the trained ResNet is an interesting but difficult problem. Attacking this question requires a fine understanding of the interaction between training dynamics and the regularity of the sequence of the weights during the gradient descent.
However, there is experimental evidence that the trained weights exhibit strong structure as a function of the layer index $k$ \citep{cohen2021scaling,bayer2022resnetsRoughPath}, and that their regularity strongly depends on the choice of initialization. 
Interesting theoretical preliminary results in the ODE case are reported in \citet{sander2022do} for a linear activation and further generalized by \citet{marion2024implicit}.
Figure \ref{fig:ex-weights-after-training} depicts this mechanism by plotting a given coordinate of $w_k$ as a function of the layer index $k$ ranging from $1$ to the depth $L=1000$, after training. Note that, in this experiment, there is no bias term in the residual layers, following the formulations from Table \ref{tab:examples}. We check that adding a bias term gives qualitatively similar plots (see Appendix~\ref{apx:experimental-setting}).

\begin{figure}[!b]
     \centering
     \begin{subfigure}[b]{0.32\textwidth}
         \centering
         \includegraphics[width=\textwidth]{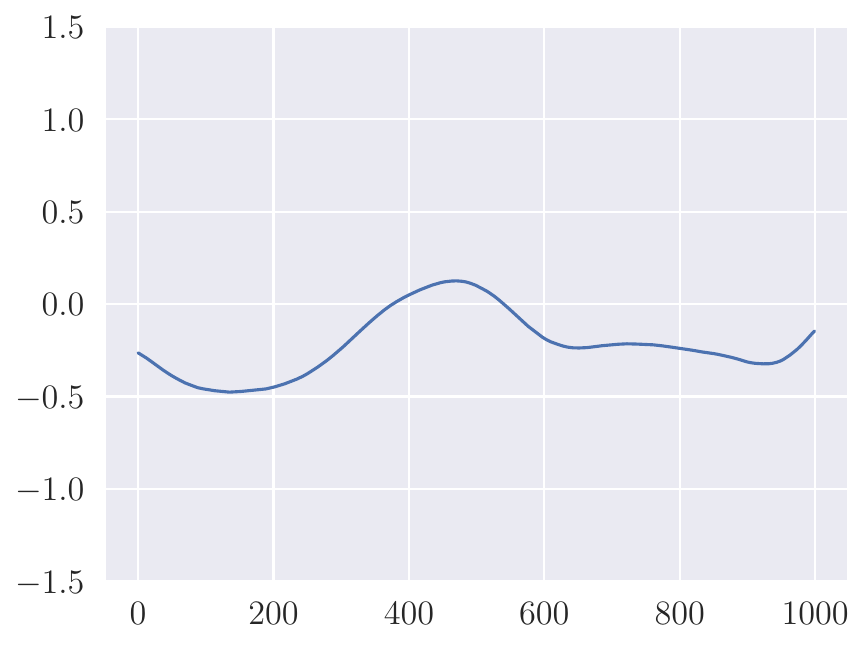}
         \caption{$\beta=1$, smooth initialization}
     \end{subfigure}
     \hfill
     \begin{subfigure}[b]{0.32\textwidth}
         \centering
         \includegraphics[width=\textwidth]{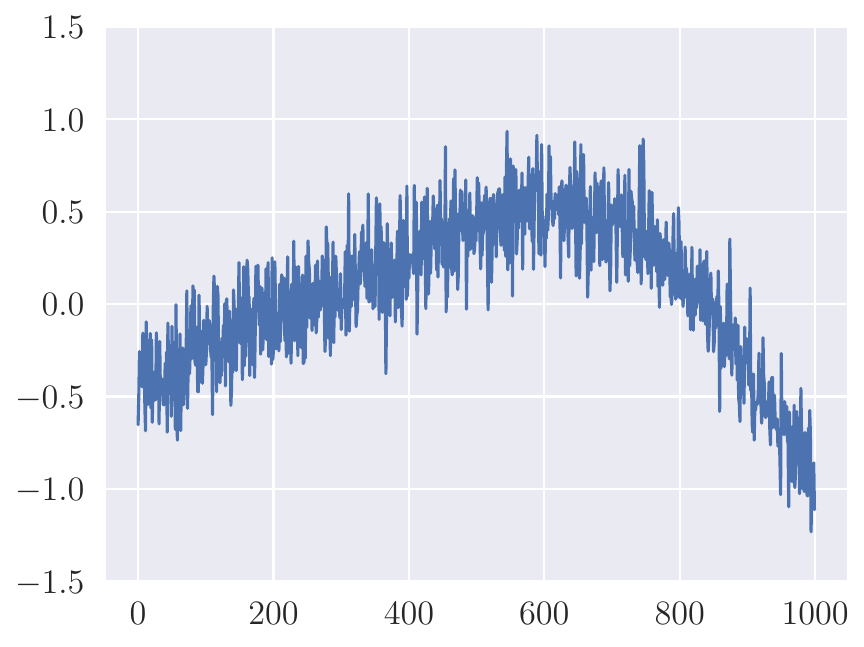}
         \caption{$\beta=1$, i.i.d.~initialization}
     \end{subfigure}
     \hfill
     \begin{subfigure}[b]{0.32\textwidth}
         \centering
         \includegraphics[width=\textwidth]{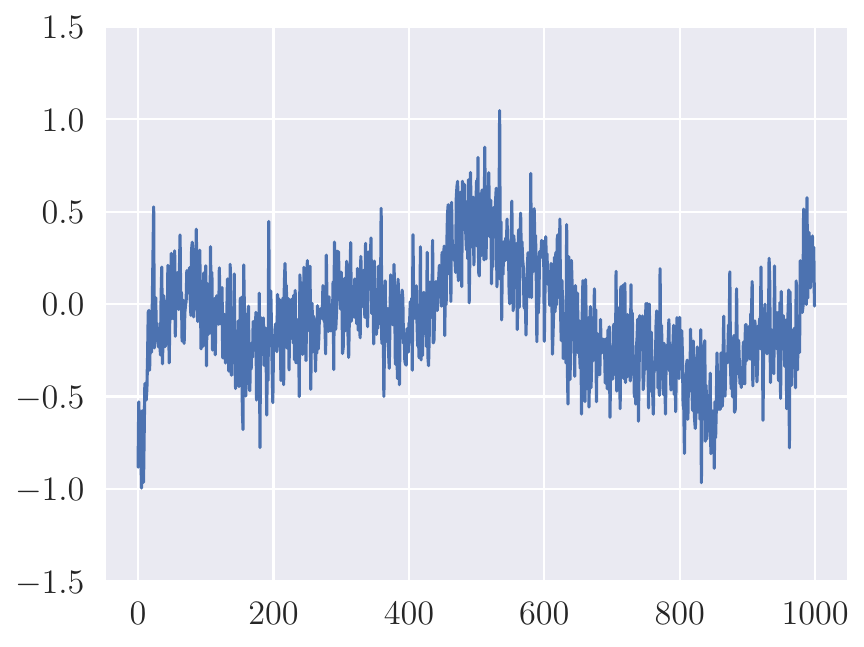}
         \caption{$\beta=\nicefrac{1}{2}$, i.i.d.~initialization}
     \end{subfigure}
     \caption{Plot of a given coordinate of $w_k$, after training, as a function of the layer index $k$ ranging from $1$ to the depth $L=1000$ for three different choices of $\beta$ and initializations.}
     \label{fig:ex-weights-after-training}
\end{figure}
To investigate the link between regularity of the weights at initialization, scaling, and performance after training, we train ResNets on the datasets MNIST \citep{deng2012mnist} and CIFAR-10 \citep{Krizhevsky2009learningmultiple}. As in Subsection~\ref{subsec:fbm-magnitude}, we initialize the ResNets with various scaling factors and weights that are increments of fractional Brownian motions with different regularities. Then, for each combination of weight initialization and scaling factor, the ResNet is trained using the Adam optimizer \citep{kingmaAdamMethodStochastic2017} for $10$ epochs. The model includes a zero-initialized bias term on each residual layer. The results in terms of accuracy are presented in Figure~\ref{fig:heatmap-trained} (light orange = good performance, blue = bad performance). We observe a pattern similar to the one of Figure~\ref{fig:heatmap}. This means that, for a given regularity, the network is unable to learn if it is initialized with a scaling too far below the critical value, which of course is connected with the gradient explosion issue discussed previously. On the other hand, and perhaps more surprisingly, the performance seems to be more or less stable in the identity region, with perhaps a small degradation in the case of CIFAR-10. This somewhat contrasts with the results from \citet{yang2017mean}, who exhibit a decrease in performance for i.i.d.~initialization and a large scaling factor $\beta$. Note however that, in the experiments reported in Figure~\ref{fig:heatmap}, we adapt the learning rate of the gradient descent on a grid by cross-validation. %
 When taking a fixed learning rate, we also observe a decrease in performance for large scaling factor $\beta$.
The interplay between the learning rate and the scaling factor is one of the keys to better assess how the performance of the trained network is connected with the scaling.
\begin{figure}[!t]
     \centering
     \begin{subfigure}[b]{0.49\textwidth}
         \centering
         \includegraphics[width=\textwidth]{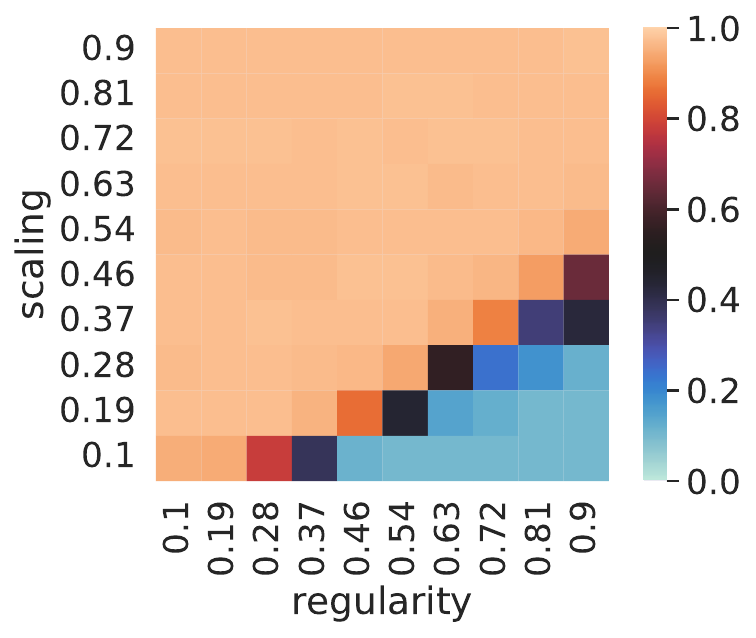}
         \caption{On MNIST}
     \end{subfigure}
     \hfill
     \begin{subfigure}[b]{0.49\textwidth}
         \centering
         \includegraphics[width=\textwidth]{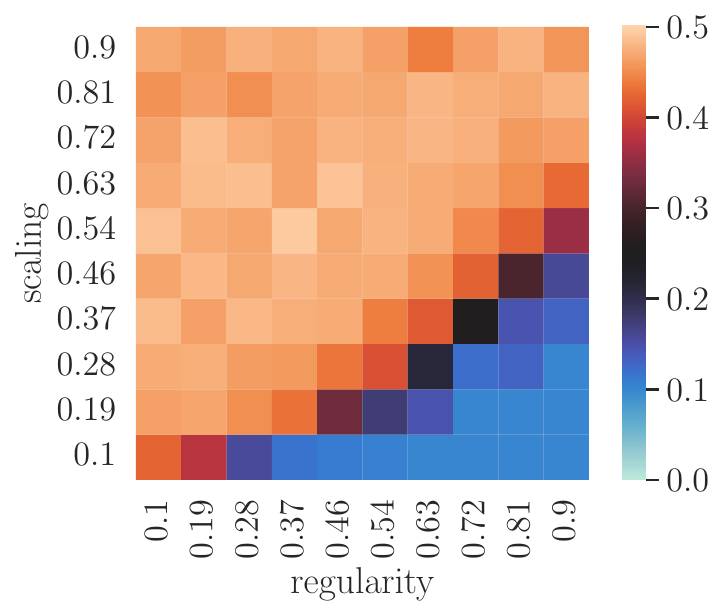}
         \caption{On CIFAR-10}
     \end{subfigure}
     \caption{Accuracy after training as a function of the regularity of the weights at initialization and scaling. For each point of the heatmap, the model was trained on a grid of learning rates, and the best performance is shown.}
     \label{fig:heatmap-trained}
\end{figure}

\acks{The authors thank anonymous referees for insightful comments. They also thank S. Schoenholz for fruitful discussion. P. Marion was supported by a grant from R\'{e}gion Île-de-France and by a Google PhD Fellowship award.}

\appendix
\section{Proofs}

Throughout the proofs, the $i$-th coordinate of a vector $v$ is denoted by $v_i$. Similarly, the $i$-th row of a matrix $M$ is denoted by $M_i$, and its $(i, j)$-th entry by $M_{ij}$.

\subsection{Proof of Proposition \ref{prop:standard-resnet-verifies-assumptions}}

Statement $(i)$ is clear (with $C = 1$) since, for any $h\in\R^d$,
\begin{equation*}
 \|\sigma(h)\|^2 \in \big[ a^2 \|h\|^2, b^2 \|h\|^2 \big] \subseteq \big[\frac{1}{2} \|h\|^2, \|h\|^2 \big].
\end{equation*}
With respect to statement $(ii)$, it is enough to show that for any $h \in \R^d$ and any random matrix $W$ satisfying the assumptions of the proposition, one has
\[
\frac{\|h\|^2}{2} \leqslant\Esp\big(\|\sigma(Wh)\|^2\big) \leqslant \|h\|^2 \quad \textnormal{and} \quad 
\Esp\big(\|\sigma(Wh)\|^8\big) \leqslant C\|h\|^8,
\]
as well as
\[
\Esp\big(\|\ReLU(Wh)\|^2\big) = \frac{\|h\|^2}{2} \quad \textnormal{and} \quad 
\Esp\big(\|\ReLU(Wh)\|^8\big) \leqslant C\|h\|^8.
\]
The two claims with the squared norms are consequences of Lemmas \ref{lemma:technical-proof-1} and \ref{lemma:technical-proof-2} in Appendix \ref{apx:sub-gaussian-matrices}, together with the fact that the variance of the entries of $W$ equals $1/d$. In order to prove the other two statements, first note that
$\Esp(\|\sigma(Wh)\|^8) \leqslant\Esp(\|Wh\|^8)$ and $\Esp(\|\ReLU(Wh)\|^8) \leqslant\Esp(\|Wh\|^8)$. The results are then consequences of Lemma \ref{lemma:bound-deviations-quadratic} in Appendix \ref{apx:sub-gaussian-matrices}, which states that
\[
\Esp\left(\frac{\|Wh\|^{8}}{\|h\|^{8}}\right) \leqslant 1 + \frac{384s^4}{d} + \frac{3072s^6}{d^2} \leqslant 1 + 384s^4 + 3072s^6.
\]

\subsection{Proof of Proposition \ref{prop:init-forward-high-prob}}
According to Lemma \ref{lemma:technical-proof-3} below, one has
\[\Esp \Big(\frac{\|h_L - h_0\|^2}{\|h_0\|^2}\Big) \leqslant \big(\big(1+\alpha_L^2\big)^L - 1\big).\]
But, for $L \alpha_L^2 \leqslant1$, we have
$
(1+\alpha_L^2)^L - 1 \leqslant\exp(L \alpha_L^2) - 1 \leqslant2 L \alpha_L^2.
$
Therefore,
\[
\Esp \Big(\frac{\|h_L - h_0\|^2}{\|h_0\|^2}\Big) \leqslant2 L \alpha_L^2,
\]
and the result follows from Markov's inequality.
\begin{lemma}   \label{lemma:technical-proof-3}
Consider a ResNet \eqref{eq:discrete-resnet} such that Assumptions $(A_1)$ and $(A_2)$ are satisfied. Then
\[
\bigg(\Big(1+\frac{\alpha_L^2}{2}\Big)^L - 1\bigg)\leqslant \Esp \Big(\frac{\|h_L - h_0\|^2}{\|h_0\|^2}\Big) \leqslant\Big(\big(1+\alpha_L^2\big)^L - 1\Big).
\]
\end{lemma}

\medskip \noindent \textit{Proof (Lemma \ref{lemma:technical-proof-3}).}
Taking the squared norm of the forward update rule \eqref{eq:discrete-resnet} and dividing by $\|h_0\|^2$ yields
\begin{equation}    \label{eq:norm-update}
 \frac{\|h_{k+1}\|^2}{\|h_0\|^2} = \frac{1}{\|h_0\|^2} \Big( \|h_k\|^2 + \alpha_L^2 \|V_{k+1}g(h_k, w_{k+1})\|^2 + 2 \alpha_L \langle h_k, V_{k+1}g(h_k, w_{k+1}) \rangle \Big).  
\end{equation}
We deduce by Assumptions $(A_1)$ and $(A_2)$ that
\[
\Big(1 + \frac{\alpha_L^2}{2}\Big) \Esp \Big(\frac{\|h_{k}\|^2}{\|h_0\|^2}\Big) \leqslant \Esp\Big(\frac{\|h_{k+1}\|^2}{\|h_0\|^2}\Big) \leqslant(1 + \alpha_L^2) \Esp \Big(\frac{\|h_{k}\|^2}{\|h_0\|^2}\Big).
\]
Therefore, by recurrence, we are led to
\begin{equation} \label{eq:formula-expectation-h_k}
\Big(1 + \frac{\alpha_L^2}{2}\Big)^k \leqslant \Esp\Big(\frac{\|h_{k}\|^2}{\|h_0\|^2}\Big) \leqslant(1 + \alpha_L^2)^k.   
\end{equation}
Now, observe that $h_L = h_0 + \alpha_L \sum_{k=0}^{L-1} V_{k+1}g(h_k, w_{k+1})$. Thus, we have
\begin{equation*}
 \Esp \Big(\frac{\|h_L - h_0\|^2}{\|h_0\|^2}\Big) = \alpha_L^2 \sum_{k,k'=0}^{L-1} \Esp\Big(\frac{g(h_k, w_{k+1})^\top V_{k+1}^\top V_{k'+1}g(h_{k'}, w_{k'+1})}{\|h_0\|^2}\Big).  
\end{equation*}
By conditioning on all random variables except $V_{k'+1}$ for $k < k'$ (and $V_{k+1}$ for $k > k'$), it is easy to see that the only non-zero terms are when $k=k'$. This yields
\begin{align*}
\Esp \Big(\frac{\|h_L - h_0\|^2}{\|h_0\|^2}\Big) &= \alpha_L^2 \sum_{k=0}^{L-1} \Esp\Big(\frac{\|V_{k+1}g(h_k, w_{k+1})\|^2}{\|h_0\|^2} \Big) \\
&\leqslant\alpha_L^2 \sum_{k=0}^{L-1} \Esp\Big(\frac{\|h_{k}\|^2}{\|h_0\|^2}\Big) \\
&\quad \mbox{(by Assumptions $A_1$ and $A_2$)} \\
&\leqslant\alpha_L^2 \sum_{k=0}^{L-1} (1 + \alpha_L^2)^k  \\
&\quad \mbox{(by \eqref{eq:formula-expectation-h_k})} \\
&= \big(\big(1+\alpha_L^2\big)^L - 1\big).
\end{align*}
Similarly,
\begin{align*}
\Esp \Big(\frac{\|h_L - h_0\|^2}{\|h_0\|^2}\Big)  &\geqslant \frac{\alpha_L^2}{2} \sum_{k=0}^{L-1} \Esp\Big(\frac{\|h_{k}\|^2}{\|h_0\|^2}\Big) \\
&= \Big(\Big(1+\frac{\alpha_L^2}{2}\Big)^L - 1\Big).
\end{align*}
\hfill\BlackBox\\[2mm]

\subsection{Proof of Proposition \ref{prop:init-forward-high-prob2}}
\label{proof:prop:init-forward-high-prob2}

Dividing \eqref{eq:norm-update} by $\|h_k\|^2$ and taking the logarithm leads to
\begin{equation*}
\ln(\|h_{k+1}\|^2) = \ln(\|h_k\|^2) + \ln \bigg(1 + \alpha_L^2 \frac{\|V_{k+1}g(h_k, w_{k+1})\|^2}{\|h_k\|^2} + 2 \alpha_L \Big\langle \frac{h_k}{\|h_k\|}, \frac{V_{k+1}g(h_k, w_{k+1})}{\|h_k\|} \Big\rangle \bigg).
\end{equation*}
Let 
\[Y_{k, 1} = \alpha_L^2 \frac{\|V_{k+1}g(h_k, w_{k+1})\|^2}{\|h_k\|^2}, \quad Y_{k,2} = 2 \alpha_L \Big\langle \frac{h_k}{\|h_k\|}, \frac{V_{k+1}g(h_k, w_{k+1})}{\|h_k\|} \Big\rangle,\] 
and $Y_k = Y_{k, 1} + Y_{k,2}$. The proof of Proposition \ref{prop:init-forward-high-prob2} strongly relies on the following lemma, which provides technical information on the moments of $Y_{k, 1}$ and $Y_{k,2}$. For the sake of clarity, its proof is postponed to Appendix \ref{apx:proof-prop:bounds-moments}.
\begin{lemma} \label{prop:bounds-moments}
Assume that Assumptions $(A_1)$ and $(A_2)$ are satisfied. Then

\begin{itemize}
  \begin{minipage}{0.35\linewidth}   
    \item[$(E_1)$] $\Esp(Y_{k,2}|h_k) = 0$. 
    \item[$(E_2)$] $\frac{\alpha_L^2}{2} \leqslant \Esp(Y_{k,1}|h_k) \leqslant\alpha_L^2$.
    \item[$(E_3)$] $\Esp(Y_{k,1}Y_{k,2}|h_k) = 0$.
    \item[$(E_4)$] $\Esp(Y_{k,2}^2|h_k) \leqslant4 \frac{\alpha_L^2}{d}$.
  \end{minipage}
  \hspace{1.8cm}
  \begin{minipage}{0.5\linewidth}
    \item[$(E_5)$] $\Esp(Y_{k,2}^4|h_k) \leqslant2048 \frac{s^4\alpha_L^4}{d^2}$.
    \item[$(E_6)$] $\Esp(Y_{k,1}^4|h_k) \leqslant C \left(3072 \frac{s^{6}}{d^2} + 384 \frac{s^4}{d} + 1 \right)\alpha_L^8$.
    \item[$(E_7)$] $\Esp(Y_{k,1}^2|h_k) \leqslant \sqrt{C} \left(128 \frac{s^{4}}{d} + 1 \right)\alpha_L^4$.
  \end{minipage}
\end{itemize}
\end{lemma}

For $c > 0$, we have
\begin{align*}
    \Prob\Big(\frac{\|h_L\|^2}{\|h_0\|^2} \geqslant c\Big) 
    &= \Prob \Big(\ln(\|h_L\|^2) - \ln(\|h_0\|^2) \geqslant \ln (c)\Big) \\
    &= \Prob \Big(\sum_{k=0}^{L-1} \ln(1+Y_k) \geqslant \ln (c)\Big) \\
    & \leqslant\Prob \Big(\sum_{k=0}^{L-1} Y_k \geqslant \ln (c)\Big) \\
    &\quad \mbox{(using $\ln(1+x) \leqslant x$ for $x > -1$).}
\end{align*}
Let $S = \sum_{k=0}^{L-1} Y_k - \Esp(Y_k|h_k)$. By $(E_1)$ and $(E_2)$,
\[
\sum_{k=0}^{L-1} \Esp(Y_k|h_k) \leqslant L\alpha_L^2.
\]
So, for $c > \exp(L\alpha_L^2)$,
\begin{align}
\Prob\Big(\frac{\|h_L\|^2}{\|h_0\|^2} \geqslant c\Big) 
&\leqslant \Prob \bigg(S \geqslant \ln (c) - \sum_{k=0}^{L-1} \Esp(Y_k|h_k)\bigg) \nonumber \\
&\leqslant \Prob (S \geqslant \ln (c) - L\alpha_L^2) \nonumber \\
&\leqslant \Prob \Big(S^2 \geqslant \big(\ln (c) - L\alpha_L^2\big)^2\Big) \nonumber \\
&\leqslant\frac{\Esp(S^2)}{(\ln (c) - L\alpha_L^2)^2}  \label{eq:proof-markov} \\
&\quad \mbox{(by Markov's inequality.)} \nonumber
\end{align}
It remains to upper bound $\Esp(S^2)$. To this aim, note that
\begin{align*}
\Esp(S^2) = \sum_{k=0}^{L-1} \Esp \Big( \big(Y_k - \Esp(Y_k|h_k) \big)^2 \Big) &\leqslant\sum_{k=0}^{L-1} \Esp \big(Y_k^2 \big) \\
&\leqslant4 \frac{L \alpha_L^2}{d} + 128 \sqrt{C} \frac{L \alpha_L^4s^4}{d} + \sqrt{C} L\alpha_L^4 \\
&\quad \mbox{(by $(E_3)$, $(E_4)$, and $(E_7)$)} \\
&\leqslant5 \frac{L \alpha_L^2}{d}.
\end{align*}
The last inequality is true for $\alpha_L^2 \leqslant \frac{1}{\sqrt{C}(d + 128 s^4)}$. Therefore, by inequality \eqref{eq:proof-markov}, we obtain, for $c > \exp(L\alpha_L^2)$,
\[
\Prob\Big(\frac{\|h_L\|^2}{\|h_0\|^2} \geqslant c\Big) \leqslant\frac{5L \alpha_L^2}{d\left(\ln (c) - L \alpha_L^2\right)^2}.
\]
We conclude that, for any $\delta \in (0,1)$, with probability at least $1-\delta$,
\[
\frac{\|h_L\|^2}{\|h_0\|^2} < \exp \left(L\alpha_L^2 + \sqrt{\frac{5L\alpha_L^2}{d\delta}} \right).
\]
This shows statement $(ii)$ of the proposition.

Next, to prove statement $(i)$, observe that $c > 0$,
\begin{align*}
    \Prob\Big(\frac{\|h_L\|^2}{\|h_0\|^2} \leqslant c\Big) 
    &= \Prob \Big(\ln(\|h_L\|^2) - \ln(\|h_0\|^2) \leqslant \ln(c)\Big) \\
    &= \Prob \Big(\sum_{k=0}^{L-1} \ln(1+Y_k) \leqslant \ln(c) \Big) \\
    &= \Prob \Big(\sum_{k=0}^{L-1} \ln(1+Y_k) \leqslant \ln(c) \text{ and } \forall k, Y_k \geqslant -\frac{1}{2} \Big) \\
    &\quad + \Prob \Big(\sum_{k=0}^{L-1} \ln(1+Y_k) \leqslant \ln(c) \text{ and } \exists k, Y_k < -\frac{1}{2} \Big).
\end{align*}
Using the inequality $\ln(1+x) \geqslant x - x^2$ for $x \geqslant - \nicefrac{1}{2}$, we obtain
\begin{align*}
    \Prob\Big(\frac{\|h_L\|^2}{\|h_0\|^2} \leqslant c\Big) 
    & \leqslant\Prob \Big(\sum_{k=0}^{L-1} Y_k - Y_k^2 \leqslant \ln(c) \text{ and } \forall k, Y_k \geqslant -\frac{1}{2} \Big) \\
    &\quad + \Prob \Big(\sum_{k=0}^{L-1} \ln(1+Y_k) \leqslant \ln(c) \text{ and } \exists k, Y_k < -\frac{1}{2} \Big).
\end{align*}
Thus,
\begin{equation}    \label{eq:tech-proof}
     \Prob\Big(\frac{\|h_L\|^2}{\|h_0\|^2} \leqslant c\Big) 
     \leqslant \Prob \Big(\sum_{k=0}^{L-1} Y_k - Y_k^2 \leqslant \ln(c) \Big) + \sum_{k=0}^{L-1} \Prob \Big(Y_{k,2} < -\frac{1}{2}\Big).
\end{equation}
We handle the two terms above on the right-hand side separately. For the first term, let $Z_k = Y_k - Y_k^2$ and $S = \sum_{k=0}^{L-1} Z_k - \Esp(Z_k|h_k)$.
Observe that, by $(E_1)$-$(E_4)$ and $(E_7)$, 
\begin{equation}    \label{eq:proof-lower-bound-expectation}
\sum_{k=0}^{L-1} \Esp(Z_k|h_k) \geqslant \frac{L \alpha_L^2}{2} - 4 \frac{L \alpha_L^2}{d}- 128 \sqrt{C} \frac{L \alpha_L^4s^4}{d} - \sqrt{C} L \alpha_L^4 \geqslant  \frac{3}{8}L\alpha_L^2,   
\end{equation}
where the last inequality is valid for $d \geqslant 64$ and $\alpha_L^2 \leqslant\frac{1}{16\sqrt{C}(2s^4+1)}$. Hence, for $0 < c < \exp(\nicefrac{3L\alpha_L^2}{8})$,
\begin{align*}
    \Prob \bigg(\sum_{k=0}^{L-1} Y_k - Y_k^2 \leqslant\ln c \bigg) 
&= \Prob \Big(S \leqslant \ln(c) - \sum_{k=0}^{L-1} \Esp(Z_k|h_k) \Big) \\ 
&\leqslant \Prob \Big(S \leqslant \ln(c) - \frac{3L \alpha_L^2}{8} \Big) \\ 
&\leqslant \Prob \Big(S^2 \geqslant \Big(\ln(c) - \frac{3L \alpha_L^2}{8} \Big)^2\Big) \\ 
&\leqslant \frac{\Esp(S^2)}{\big(\ln(c) - \frac{3L \alpha_L^2}{8}\big)^2} \\
&\quad \mbox{(by Markov's inequality.)}
\end{align*}
Using the $c_r$-inequality $(a+b)^n \leqslant 2^{n-1} (a^n + b^n)$ respectively for $n=2$ and $n=4$, we see that
\begin{align*}
\Esp(S^2) &= \sum_{k=0}^{L-1} \Esp \Big( \big(Z_k - \Esp(Z_k|h_k) \big)^2 \Big) 
\leqslant\sum_{k=0}^{L-1} \Esp \big(Z_k^2 \big) 
\leqslant2 \sum_{k=0}^{L-1} \Esp \big(Y_k^2 \big) + \Esp \big(Y_k^4 \big) \\
&\leqslant2 \sum_{k=0}^{L-1}
\Esp(Y_{k,1}^2) + \Esp(Y_{k,2}^2)  + 2 \Esp(Y_{k,1}Y_{k,2}) + 8 \Esp(Y_{k,1}^4) + 8 \Esp(Y_{k,2}^4).
\end{align*}
By $(E_3)$-$(E_7)$, it is easy to verify that, for $d \geqslant 64$ and $\alpha_L^2 \leqslant\frac{1}{(\sqrt{C}s^4/16 + 2 \sqrt{C} + 8 s^4)d}$,
\begin{align*}
\Esp(S^2)
&\leqslant10 \frac{L \alpha_L^2}{d}.
\end{align*}
This shows that, for $c < \exp(\nicefrac{3L\alpha_L^2}{8})$,
\[
\Prob \Big(\sum_{k=0}^{L-1} Y_k - Y_k^2 \leqslant \ln(c) \Big)
\leqslant\frac{10 L \alpha_L^2}{d\big(\ln(c) - \frac{3L \alpha_L^2}{8}\big)^2}.
\]
To conclude the proof, it remains to upper bound the second term of inequality \eqref{eq:tech-proof}. According to inequality~\eqref{eq:general-bound-large-deviation} in the proof of Lemma \ref{prop:bounds-moments} (with $t=\nicefrac{1}{2}$), one has
\[
\sum_{k=0}^{L-1} \Prob \Big(Y_{k,2} < -\frac{1}{2}\Big) \leqslant2 L \exp \Big(- \frac{d}{64 \alpha_L^2 s^2} \Big).    
\]
Putting everything together, we are led to
\[
\Prob\Big(\frac{\|h_L\|^2}{\|h_0\|^2} \leqslant c\Big) \leqslant \frac{10L \alpha_L^2}{d\big(\ln(c) - \frac{3L \alpha_L^2}{8}\big)^2} + 2 L \exp \Big(- \frac{d}{64 \alpha_L^2s^2} \Big).
\]
Take $\delta \in (0, 1)$. Then, if $2 L \exp \big(- \frac{d}{64 \alpha_L^2s^2} \big) \leqslant\frac{\delta}{11}$, with probability at least $1-\delta$,
\[
\frac{\|h_L\|^2}{\|h_0\|^2} > \exp \bigg(\frac{3L\alpha_L^2}{8} - \sqrt{\frac{11L\alpha_L^2}{d\delta}}\bigg).
\]
Notice that this inequality is valid under the assumption  $\alpha_L^2 \leqslant\frac{2}{(\sqrt{C}s^4 + 4 \sqrt{C} + 16 s^4)d}$.

\subsection{Proof of Corollary \ref{corollary:forward}}

Statement $(i)$ is a consequence of Proposition \ref{prop:init-forward-high-prob}, whereas $(ii)$ is a consequence of Proposition~\ref{prop:init-forward-high-prob2}~$(i)$. 
The latter is valid under the conditions $d \geqslant 64$ and $\alpha_L \leqslant \frac{2}{(\sqrt{C}s^4 + 4 \sqrt{C} + 16 s^4)d}$, which is automatically satisfied for all $L$ large enough. Furthermore, an inspection of the proof of Proposition \ref{prop:init-forward-high-prob2} reveals that the divergence in high probability of $\|h_L\|$ can be proved under the relaxed assumption $d \geqslant 9$. Indeed, the main constraint on $d$ comes from the lower bound~\eqref{eq:proof-lower-bound-expectation}, where one needs to make sure that $\frac{L\alpha_L^2}{2} - 4 \frac{L\alpha_L^2}{d} > 0$, which is the case for $d=9$.

To prove $(iii)$, we use a union bound on both statements of Proposition \ref{prop:init-forward-high-prob2}.
\subsection{Proof of Proposition \ref{prop:standard-resnet-verifies-assumptions-gradients}}

The first claim follows from the observation that
\[
\frac{\partial g(h_k, w_{k+1})}{\partial h} q_{k} = 
\begin{pmatrix}
\sigma'(h_{k,1}) & 0 & \dots & 0\\
0 & \sigma'(h_{k,2}) & \dots & 0 \\
\vdots & \vdots & \ddots & \vdots \\
0 & 0 & \dots & \sigma'(h_{k,d}) \\
\end{pmatrix}
q_k,
\]
from $(A_1)$, and from the assumption on $\sigma'$.

Let us now prove $(ii)$. In the rest of the proof, the subscript $k$ is ignored to lighten the notation. Observe that
\[
\frac{\partial g(h, w)}{\partial h} q = V
\begin{pmatrix}
\sigma'(\langle W_{1}, h\rangle) & 0 & \dots & 0\\
0 & \sigma'(\langle W_{2}, h\rangle) & \dots & 0 \\
\vdots & \vdots & \ddots & \vdots \\
0 & 0 & \dots & \sigma'(\langle W_{d}, h\rangle) \\
\end{pmatrix}
W q.
\]
 Denote by $D$ the matrix in the middle of the right-hand side. 
 Then
\begin{align*}
\Esp \bigg( \Big\| \frac{\partial g(h, w)}{\partial h} q \Big\|^2  \Big| h, q \bigg) = \Esp\big(\|V D W q\|^2 | h, q\big)
&= \Esp\big(\|D W q\|^2 | h, q\big) \\
&\quad \mbox{(by $(A_1)$)}
\end{align*}
For model \texttt{res-2}, we have
\begin{equation*}
\Esp \bigg( \Big\| \frac{\partial g(h, w)}{\partial h} q \Big\|^2  \Big| h, q \bigg) = \Esp \bigg( \sum_{i=1}^d \Big(\sum_{j=1}^d W_{ij} q_{j} \Big)^2 \sigma'(\langle W_i, h\rangle) \bigg| h,q \bigg).
\end{equation*}
The conclusion follows from the hypothesis that $a \leqslant \sigma' \leqslant b$ and $\Esp\big(\|W q\|^2 | q\big) 
= \|q\|^2$.
For model \texttt{res-3}, we have
\begin{equation*}
\Esp \bigg( \Big\| \frac{\partial g(h, w)}{\partial h} q \Big\|^2  \Big| h, q \bigg) = \Esp \bigg( \sum_{i=1}^d \Big(\sum_{j=1}^d W_{ij} q_{j} \Big)^2 \mathbf{1}_{\sum_{j=1}^d W_{ij} h_{j} \geqslant 0} \bigg| h,q \bigg).
\end{equation*}
Since the $(W_{ij})_{1 \leqslant i,j \leqslant d}$ are centered random variables, we conclude that
\begin{equation*}
\Esp \bigg( \Big\| \frac{\partial g(h, w)}{\partial h} q \Big\|^2  \Big| h, q \bigg) 
= \frac{1}{2} \Esp \bigg( \sum_{i=1}^d \Big(\sum_{j=1}^d W_{ij} q_{j} \Big)^2  \bigg| q \bigg)
= \frac{1}{2} \Esp\big(\|W q\|^2 | q\big) 
= \frac{\|q\|^2}{2}.
\end{equation*}

\subsection{Proof of Proposition \ref{prop:gradients}}

Letting $b = p_L / \|p_L\|$, as in Assumption $(A_3)$, and taking expectation in \eqref{eq:forward-diff-main}, we obtain
\begin{align}
\Esp \bigg( \frac{\|p_0\|^2}{\|p_L\|^2}  \bigg)
= \Esp (|b^\top q_L(z)|^2) &= \frac{1}{d} \Esp (\|q_L(z)\|^2) \label{eq:formula-expectation-gradients} \\
&\quad \mbox{(by $(A_3)$)}. \nonumber %
\end{align}
The rest of the proof is similar to the proof of Proposition~\ref{prop:init-forward-high-prob}. From~\eqref{eq:general-rel-q_k}, we have
\[
\|q_{k+1}(z)\|^2 = \|q_k(z)\|^2 + \alpha_L^2 \Big\| V_{k+1} \frac{\partial g(h_k, w_{k+1})}{\partial h} q_k(z) \Big\|^2 + 2 \alpha_L \big\langle q_k(z), V_{k+1} \frac{\partial g(h_k, w_{k+1})}{\partial h} q_k(z) \big\rangle.
\]
By independence of $V_{k+1}$ from $q_{k}(z)$ and $\frac{\partial g(h_k, w_{k+1})}{\partial h}$,
\begin{align*}
\Esp\Big( \Big\langle q_k(z), V_{k+1} \frac{\partial g(h_k, w_{k+1})}{\partial h} q_k(z) \Big\rangle \Big) = 0.
\end{align*}
Next,
\begin{align*}
\Esp \Big( \Big\| V_{k+1} \frac{\partial g(h_k, w_{k+1})}{\partial h} q_{k}(z) \Big\|^2 \Big) &= \Esp \bigg(\Esp \Big( \Big\| V_{k+1} \frac{\partial g(h_k, w_{k+1})}{\partial h} q_{k}(z) \Big\|^2 \Big| h_k, w_{k+1}, q_{k}(z) \Big) \bigg) \\
&= \Esp \left( \Big\| \frac{\partial g(h_k, w_{k+1})}{\partial h} q_{k}(z) \Big\|^2  \right) \\
&\quad \mbox{(by $(A_1)$)} \\
&= \Esp \left( \Esp \left( \Big\| \frac{\partial g(h_k, w_{k+1})}{\partial h} q_{k}(z) \Big\|^2  \Big| h_k, q_{k}(z) \right) \right).
\end{align*}
By Assumption $(A_3)$, we are led to
\[
\big(1 + \frac{1}{2} \alpha_L^2\big) \Esp(\|q_{k}(z)\|^2) \leqslant\Esp(\|q_{k+1}(z)\|^2) \leqslant(1 + \alpha_L^2) \Esp(\|q_{k}(z)\|^2),
\]
and thus, by induction, since $q_0(z) = z$ and $\Esp(\|z\|^2) = d$,
\[
d\big(1 + \frac{1}{2} \alpha_L^2\big)^k \leqslant\Esp(\|q_k(z)\|^2) \leqslant d(1 + 4\alpha_L^2)^k.
\]
In particular, for $k=L$,
\[
d\big(1 + \frac{1}{2} \alpha_L^2\big)^L  \leqslant\Esp(\|q_L(z)\|^2) \leqslant d(1 + \alpha_L^2)^L.
\]
Therefore, by \eqref{eq:formula-expectation-gradients},
\begin{equation*}
\big(1 + \frac{1}{2} \alpha_L^2\big)^L \leqslant \Esp \bigg( \frac{\|p_0\|^2}{\|p_L\|^2}  \bigg) \leqslant (1 + \alpha_L^2)^L.  
\end{equation*}
To finish the proof, observe that
\[
\frac{1}{\|p_L\|} (p_0 - p_L)^\top z = b^\top (q_L(z) - z).
\]
Using arguments similar to \eqref{eq:formula-expectation-gradients}, we may write
\[
\Esp \bigg( \frac{\|p_0-p_L\|^2}{\|p_L\|^2} \bigg)
= \frac{1}{d} \Esp \Big(\frac{\|q_L(z) - z\|^2}{\|z\|^2} \Big).
\]
Now, upon noting that $q_L(z) - z = q_L(z) - q_0(z) = \alpha_L \sum_{k=0}^{L-1} V_{k+1} \frac{\partial g(h_k, w_{k+1})}{\partial h} q_{k}(z)$,
\begin{align*}
\Esp(\|q_L(z) - z\|^2) &= \alpha_L^2 \sum_{k, k'=0}^{L-1} \Esp \left( q_{k}(z)^\top  \frac{\partial g(h_k, w_{k+1})^\top}{\partial h} V_{k+1}^\top V_{k'+1} \frac{\partial g(h_k', w_{k'+1})}{\partial h} q_{k'}(z) \right) \\
&= \alpha_L^2 \sum_{k=0}^{L-1} \Esp \left( \bigg\| V_{k+1} \frac{\partial g(h_k, w_{k+1})}{\partial h} q_{k}(z) \bigg\|^2 \right) \\
&\leqslant d\alpha_L^2 \sum_{k=0}^{L-1} (1 + \alpha_L^2)^{k}  \\
&= d \big((1 + \alpha_L^2)^L - 1 \big)
\leqslant d \big(\exp(L\alpha_L^2) - 1 \big) \leqslant 2 d L\alpha_L^2,
\end{align*}
for $L \alpha_L^2 \leqslant 1$. Note that the second equality is obtained by conditioning on every random variable except $V_{k'+1}$ for $k < k'$ (and $V_{k+1}$ for $k > k'$). Finally, by using Markov's inequality, we conclude that, for any $\varepsilon > 0$,
\[
\Prob \big(\|p_0 - p_L\|^2 \geqslant \varepsilon \|p_L\|^2 \big) \leqslant\frac{2 L\alpha_L^2}{\varepsilon}.
\]

\subsection{Proof of Proposition \ref{prop:gradients2}}

The proof of Proposition \ref{prop:gradients} reveals that
\[
\Esp \bigg( \frac{\|p_0-p_L\|^2}{\|p_L\|^2} \bigg)
\leqslant (1 + \alpha_L^2)^L - 1.
\]
Using similar arguments, one has
\begin{align*}
\Esp \bigg( \frac{\|p_0-p_L\|^2}{\|p_L\|^2} \bigg) = \frac{1}{d} \Esp \Big(\frac{\|q_L(z) - z\|^2}{\|z\|^2} \Big) \geqslant \alpha_L^2 \sum_{k=0}^{L-1} \big(1 + \frac{1}{2} \alpha_L^2\big)^{k} = \big(1 + \frac{1}{2} \alpha_L^2\big)^L - 1.
\end{align*}

\subsection{Proof of Corollary \ref{cor:gradients2}}

The first statement is an immediate consequence of Proposition \ref{prop:gradients}. The second one is a consequence of Proposition \ref{prop:gradients2} and the fact that, for $\beta < 1$,
\[
\Big(1 + \frac{1}{L^\beta}\Big)^L = \exp \Big( L \ln \Big( 1 + \frac{1}{L^\beta} \Big) \Big) \sim \exp \big( L^{1 - \beta} \big) \rightarrow \infty.
\]
Finally, $(iii)$ follows from Proposition \ref{prop:gradients2}.

\subsection{Proof of Proposition \ref{prop:resnet_convergence_sde_init}}

The proposition is a consequence of \citet[][Theorems 4.5.3 and 10.2.2]{kloeden1992numerical} for the SDE
\[    dH_t^\top = \sqrt{\frac{1}{d}} \sigma(H_t^\top)dB_t. \]
Letting $a(h, t)=0$ and $b(h,t)=\sqrt{\frac{1}{d}} \sigma(h)$, we need to check the following assumptions:
\begin{itemize}
    \item[$(H_1)$] The functions $a(\cdot, \cdot)$ and $b(\cdot, \cdot)$ are jointly measurable on $\R^d \times [0,1]$.
    \item[$(H_2)$] There exists a constant $C_1 > 0$ such that, for any $x,y \in \R^d$, $t \in [0,1]$,
    \begin{equation*}
        \| a(x,t) - a(y,t)\| + \|b(x,t) - b(y,t)\| \leqslant C_1 \|x-y\|.
    \end{equation*}
    \item[$(H_3)$] There exists a constant $C_2>0$ such that, for any $x \in \R^d$, $t \in [0,1]$,
    \begin{equation*}
        \|a(x,t)\| + \|b(x,t)\| \leqslant C_2 (1 + \|x\|).
    \end{equation*}
    \item[$(H_4)$] $\esp \big(\|H_0\|^2\big) < \infty$.
    \item[$(H_5)$] There exists a constant $C_3 > 0$ such that, for any $x \in \R^d$, $s,t \in [0,1]$,
    \begin{equation*}
        \|a(x,t) - a(x,s)\| + \|b(x,t) - b(x,s) \| \leqslant C_3(1 + \|x\|)|t-s|^{\nicefrac{1}{2}}.
    \end{equation*}
\end{itemize}
Assumptions $(H_1)$, $(H_4)$, and $(H_5)$ readily follow from the definitions. Assumption $(H_2)$ is true since $\sigma$ is Lipschitz continuous, and $(H_3)$ follows from
\[\|\sigma(x)\| \leqslant b \|x\| \leqslant\|x\| \leqslant1 + \|x\|. \]

\subsection{Proof of Proposition \ref{prop:neural-ode}}

Let $\psi: \R^d \times [0,1] \to \R^d$ be defined for any $h\in \R^d$, $t \in [0,1] $, by $\psi(h,t) =\mathscr{V}_t g(h, \mathscr{W}_t)$. With this notation, the ODE \eqref{eq:neural-ode} is equivalent to the initial value problem
\[dH_t = \psi(H_t, t) dt, \quad H_0 = Ax.\]
By Assumptions $(A_5)$ and $(A_6)$, $\psi$ is Lipschitz continuous in its first argument, in the sense that there exists $K > 0$ (which may depend on the realization of $\mathscr{V}$ and $\mathscr{W}$) such that, for all $h,h' \in \R^d, t \in [0, 1]$,
$$\|\psi(h,t) - \psi(h',t)\| \leqslant K \|h-h'\|.$$
In addition, it is continuous in its second one. 
Thus, according to the Picard-Lindel\"{o}f theorem (Theorem \ref{thm:picard-lindelof} in Appendix \ref{sec:picard-lindelof}), this is enough to show that the neural ODE \eqref{eq:neural-ode} has a unique solution on $[0,1]$. Note that the solution $H$ is continuous on $[0,1]$ and is therefore bounded by a constant $M>0$. 
 
In order to prove the approximation bound of Proposition \ref{prop:neural-ode}, we start by proving that both $\psi$ and $H$ are Lipschitz continuous in $t$. Under $(A_5)$ and $(A_6)$, this is clear for $\psi$ since $H$ is bounded. %
Moreover,
for any $[s,t] \subset [0,1]$, we have
\begin{align*}
    \|H_t - H_s\| =\Big\| \int_s^t \psi(H_u, u)du \Big\| &\leqslant  \int_s^t \|\psi(H_u, u)\|du \\
    		&\leqslant  (t-s) \sup_{ \substack{u \in [0,1]\\ h \in \R^d, \|h\| \leqslant M}}\|\psi(h, u) \|.
\end{align*}

Now, let $K_1$ and $K_2$ denote the Lipschitz constants of $\psi$ (in both arguments) %
and $H$ respectively, and, for any $0 \leqslant k \leqslant L$, let $t_k = \nicefrac{k}{L}$.
Then we have, for $k \geqslant 1$,
\begin{align*}
    &\|H_{t_k} - h_k  \|\\
    &\quad= \big\|H_{t_{k-1}} + \int_{t_{k-1}}^{t_k} \psi(H_u, u)du -h_{k-1} - \frac{1}{L} \psi(h_{k-1}, t_{k-1}) \big\| \\
    &\quad \leqslant \| H_{t_{k-1}} -h_{k-1}\| + \int_{t_{k-1}}^{t_k} \|\psi(H_u, u) - \psi(h_{k-1}, t_{k-1})\|du \\
    & \quad\leqslant \| H_{t_{k-1}} -h_{k-1}\| + K_1 \int_{t_{k-1}}^{t_k}\|H_u-h_{k-1}\|du + K_1 \int_{t_{k-1}}^{t_k}|u-t_{k-1}|du \\
    & \quad\leqslant \Big(1 + \frac{K_1}{L} \Big)\| H_{t_{k-1}} -h_{k-1}\| + K_1 \int_{t_{k-1}}^{t_k}\|H_u-H_{t_{k-1}}\|du + K_1 \int_{t_{k-1}}^{t_k}|u-t_{k-1}|du \\
    & \quad\leqslant \Big(1 + \frac{K_1}{L} \Big)\| H_{t_{k-1}} -h_{k-1}\| + (K_2 +1)K_1\int_{t_{k-1}}^{t_k}|u-t_{k-1}|du \\
    &\quad = \Big(1 + \frac{K_1}{L} \Big)\| H_{t_{k-1}} -h_{k-1}\| + \frac{(K_2 +1)K_1}{2L^2}.
\end{align*}
By recurrence, we obtain
\begin{align*}
     \|H_{t_k} - h_k  \|& \leqslant \sum_{j=0}^{k-1} \Big( 1 + \frac{K_1}{L}\Big)^{j} \times \frac{(K_2 +1)K_1}{2L^2} \leqslant L \Big( 1 + \frac{K_1}{L}\Big)^{L} \times \frac{(K_2 +1)K_1}{2L^2} \\
     &\leqslant e^{K_1}\frac{(K_2 +1)K_1}{2L},
\end{align*}
which concludes the proof.

\subsection{Proof of Proposition \ref{prop:stability-neural-ode}}

Starting from \eqref{eq:discrete-resnet} and using Assumption $(A_6)$, one easily obtains the existence of $C_1$ and $C_2$ (whose values depend on the realization of $\mathscr{V}$ and $\mathscr{W}$) such that
\[
\|h_{k+1}\| \leqslant (1 + C_1\alpha_L) \|h_k\| + C_2 \alpha_L.
\]
By recurrence,
\[
\|h_{k+1}\| \leqslant (1 + C_1\alpha_L)^k \Big( \|h_0\| + \frac{C_2}{C_1} \Big).
\]
Hence, using $\alpha_L \leqslant \nicefrac{1}{L}$,
\[
\|h_{k+1}\| \leqslant \exp(C_1) \Big( \|h_0\| + \frac{C_2}{C_1} \Big).
\]
Since $g$ is Lipschitz continuous on compact sets, it is bounded on every ball of $\R^d \times \R^p$. The result is then a consequence of the identity
\[
h_L - h_0 = \alpha_L \sum_{k=0}^{L-1} V_{k+1} g(h_k, w_{k+1}),
\]
since we showed that each term in the sum is bounded by some constant $C_3 > 0$, independent of $L$ and $k$. Hence we have that
\[
\|h_L - h_0\| \leqslant C_3 L \alpha_L = C_3 L^{1-\beta},
\]
yielding the results depending on the value of $\beta$.

\subsection{Proof of Proposition \ref{prop:explosion-neural-ode}}

In the linear case, \eqref{eq:discrete-resnet} can be written
\[
h_{k+1} = h_k + \alpha_L V_{k+1} h_k, \quad 0 \leqslant k \leqslant L-1.
\]
Take $y$ a unit-norm eigenvector of $\mathscr{V}_0^\top$ with associated eigenvalue $\lambda > 0$. Then
\begin{align*}
\langle h_{k+1}, y \rangle &= \langle h_k + \alpha_L V_{k+1} h_k, y \rangle \\
&= \langle h_k, y \rangle + \alpha_L \langle  h_k, V_{k+1}^\top y \rangle \\
&= \langle h_k, y \rangle + \lambda \alpha_L \langle h_k, y \rangle + \alpha_L  \langle h_k, (V_{k+1} - \mathscr{V}_0)^\top y \rangle.
\end{align*}
Since $\mathscr{V}$ is Lipschitz and $V_{k+1} = \mathscr{V}_{k+1/L}$, there exists $c$ such that $\|V_{k+1} - \mathscr{V}_0\| \leqslant c \frac{k+1}{L}$. Hence
\[
|\langle h_{k+1}, y \rangle| \geqslant (1 + \lambda \alpha_L) | \langle h_k, y \rangle | -c \alpha_L \frac{k+1}{L} \|h_k\|.
\]
Then, by recurrence,
\begin{align*}
|\langle h_L, y \rangle| &\geqslant (1 + \lambda \alpha_L)^{L} |\langle h_0, y \rangle| - c \frac{\alpha_L}{L}  \sum_{k=0}^{L-1} (k+1) (1 + \lambda \alpha_L)^{k} \|h_k\| \\
&\geqslant (1 + \lambda \alpha_L)^{L} |\langle h_0, y \rangle| - c \alpha_L (1 + \lambda \alpha_L)^L \max_k \|h_k\|.
\end{align*}
Let $M = \frac{|\langle h_0, y \rangle|}{2c \alpha_L}$, and suppose that $\|h_k\| \leqslant M$ for all $0 \leqslant k \leqslant L$. Then
\begin{align*}
\|h_L\| 
&\geqslant |\langle h_L, y \rangle| \\
&\quad \mbox{(by the Cauchy-Schwartz inequality)} \\
&\geqslant (1 + \lambda \alpha_L)^{L} \big( |\langle h_0, y \rangle| - c M \alpha_L \big).    
\end{align*}
Then, for $\lambda \alpha_L \leqslant 1$,
\[
\|h_L\| \geqslant \frac{1}{2} (1 + \lambda \alpha_L)^L |\langle h_0, y \rangle| \geqslant \frac{1}{2} \exp \big(\frac{\lambda L\alpha_L }{2}\big) |\langle h_0, y \rangle|.
\]
Thus, since $L \alpha_L = L^{1-\beta}$, we have that $\|h_L\| \rightarrow \infty$, which contradicts our assumption that $\|h_k\| \leqslant M$ for all $0 \leqslant k \leqslant L$. We deduce that, for all $L$ large enough,
\[
\max_k \|h_k\| > \frac{|\langle h_0, y \rangle|}{2c\alpha_L} \xrightarrow[]{L\rightarrow\infty} \infty.
\]
Furthermore,
\[
\max_k \frac{\|h_k-h_0\|}{\|h_0\|} > \frac{|\langle h_0, y \rangle|}{2c\|h_0\|\alpha_L} - 1 \xrightarrow[]{L\rightarrow\infty} \infty.
\]

\subsection{Proof of Lemma \ref{prop:bounds-moments}}   \label{apx:proof-prop:bounds-moments}

$(E_1)$ and $(E_2)$ are simple consequences of Assumptions $(A_1)$ and $(A_2)$.

To prove $(E_3)$, let $f(h_k, w_{k+1}) = V_{k+1} g(h_k, w_{k+1})$. Then 
\begin{align*}
 \Esp(Y_{k,2}Y_{k,1}|h_k) &= \frac{1}{\|h_k\|^4} \Esp\big(\|f(h_k, w_{k+1})\|^2 \langle h_k, f(h_k, w_{k+1}) \rangle \big| h_k \big) \\
 &= \Esp \Big( \sum_{i=1}^d \sum_{j=1}^d f(h_k, w_{k+1})_i^2 (h_k)_j f(h_k, w_{k+1})_j \Big| h_k \Big).    
\end{align*}
It is easy to verify that, under Assumption $(A_1)$, each term of the sum above has zero expectation. This shows $(E_3)$.

To establish $(E_4)$, we start by noting that
\begin{align*}
\Esp \Big(  \Big\langle \frac{h_k}{\|h_k\|}, \frac{f(h_k, w_{k+1})}{\|h_k\|} \Big\rangle^2 \Big| h_k \Big) &= \frac{1}{\|h_k\|^4} \Esp\big(h_k^\top f(h_k, w_{k+1}) f(h_k, w_{k+1})^\top h_k | h_k \big) \\
&= \frac{1}{\|h_k\|^4} h_k^\top \Esp \big(f(h_k, w_{k+1}) f(h_k, w_{k+1})^\top | h_k \big) h_k.
\end{align*}
Clearly, $\Esp(f(h_k, w_{k+1})_i f(h_k, w_{k+1})_j) = 0$ for $i \neq j$. Since, furthermore, the coordinates of $f(h_k, w_{k+1})$ are identically distributed conditionally on $h_k$, we obtain
\[
\Esp \big(f(h_k, w_{k+1}) f(h_k, w_{k+1})^\top | h_k \big) = \frac{1}{d} \Esp(\|f(h_k, w_{k+1})\|^2 | h_k) I_d.
\]
Thus,
\[
\Esp \Big( \Big\langle \frac{h_k}{\|h_k\|}, \frac{f(h_k, w_{k+1})}{\|h_k\|} \Big\rangle^2 \Big| h_k \Big) = \frac{1}{d\|h_k\|^4}  \Esp(\|f(h_k, w_{k+1})\|^2 | h_k) h_k^\top h_k \leqslant\frac{1}{d},
\]
by Assumptions $(A_1)$ and $(A_2)$.

To prove $(E_5)$, let $\varphi = \frac{\langle V_{k+1} g(h_k, w_{k+1}), h_k \rangle}{\|g(h_k, w_{k+1})\| \|h_k\|}$. Then, for any $t>0$,
\begin{align*}
\Prob (|Y_{k,2}| > t) 
&= \Prob \Big( |\varphi| > \frac{t\|h_k\|}{2 \alpha_L \|g(h_k, w_{k+1})\|}\Big) \\
&= \Esp \bigg( \Prob \Big( |\varphi| > \frac{t\|h_k\|}{2 \alpha_L \|g(h_k, w_{k+1})\|} \Big| h_k, w_{k+1} \Big) \bigg).
\end{align*}
So, by Lemma \ref{lemma:bound-deviations-linear} in Appendix \ref{apx:sub-gaussian-matrices},
\begin{align*}
\Prob (|Y_{k,2}| > t)
&\leqslant\Esp \bigg( 2 \exp \Big(- \frac{dt^2\|h_k\|^2}{16 \alpha_L^2s^2 \|g(h_k, w_{k+1})\|^2} \Big) \bigg) \\
&= \Esp \bigg( \Esp \bigg( 2 \exp \Big(- \frac{dt^2\|h_k\|^2}{16 \alpha_L^2s^2 \|g(h_k, w_{k+1})\|^2} \Big) \Big| h_k \bigg) \bigg) \\
&\leqslant\Esp \bigg( 2 \exp \Big(- \frac{dt^2\|h_k\|^2}{16 \alpha_L^2 s^2 \Esp(\|g(h_k, w_{k+1})\|^2 | h_k)} \Big) \bigg),
\end{align*}
by Jensen's inequality. Finally, using Assumption $(A_2)$, we deduce that
\begin{equation}    \label{eq:general-bound-large-deviation}
\Prob (|Y_{k,2}| > t) \leqslant\Esp \bigg( 2 \exp \Big(- \frac{dt^2}{16 \alpha_L^2 s^2} \Big) \bigg) = 2 \exp \Big(- \frac{dt^2}{16 \alpha_L^2 s^2} \Big).
\end{equation}
In particular, for all $q \geqslant 1$ \citep[see, e.g.,][]{Pauwels2020},
\[
\Esp(Y_{k,2}^{2q}) \leqslant q! \Big( \frac{32s^2\alpha_L^2}{d} \Big)^q.
\]
The result is obtained by taking $q=2$.

Finally, $(E_6)$ and $(E_7)$ are consequences of Lemma \ref{lemma:bound-deviations-quadratic} in Appendix \ref{apx:sub-gaussian-matrices}.

\section{Concentration of Sub-Gaussian Random Matrices} \label{apx:sub-gaussian-matrices}

In this appendix, we are interested in concentration of moments of sub-Gaussian matrices. We begin by two simple lemmas on second-order moments of random matrix-vector products.

\begin{lemma}   \label{lemma:technical-proof-1}
Let $W \in \R^{d \times d}$ be a matrix whose entries are centered i.i.d.~random variables, with finite variance, and let $\sigma$ be an activation function such that, for all $x \in \R, a |x| \leqslant|\sigma(x)| \leqslant b |x|$, $\nicefrac{1}{\sqrt{2}} \leqslant a < b \leqslant1$. Then, for any $x \in \R^d$,
\[
\frac{1}{2} \Esp\big(\|Wx\|^2\big) \leqslant\Esp\big(\|\sigma(Wx)\|^2\big) \leqslant \Esp\big(\|Wx\|^2\big)
\quad \textnormal{and} \quad \Esp\big(\|\ReLU(Wx)\|^2\big) = \frac{1}{2} \Esp\big(\|Wx\|^2\big).
\]
\end{lemma}

\begin{proof}
The first part is a consequence of the assumption on $\sigma$. To prove the equality, let $X_i = \sum_{j=1}^d W_{ij} x_j$. Then
\[
\Esp\big(\|\ReLU(Wx)\|^2\big) = \Esp \bigg( \sum_{i=1}^d \Big(\sum_{j=1}^d W_{ij} x_j \Big)^2 \mathbf{1}_{\sum_{j=1}^d W_{ij} x_j \geqslant 0}\bigg) = \Esp \Big( \sum_{i=1}^d X_i^2 \mathbf{1}_{X_i \geqslant 0} \Big).
\]
Since the $(W_{ij})_{1 \leqslant j \leqslant d}$ are centered and independent random variables, $X_i$ is also centered. Hence $\Esp(X_i^2 \mathbf{1}_{X_i \geqslant 0}) = 1/2\Esp(X_i^2)$, which concludes the proof.
\end{proof}

\begin{lemma}  \label{lemma:technical-proof-2} \label{lemma:variance-random-matrix}
Let $W \in \R^{d \times d}$ be a matrix whose entries are centered i.i.d.~random variables, with finite variance $s^2$. Then, for any $x \in \R^d$, $\Esp\big(\|Wx\|^2\big) = s^2d\|x\|^2$.
\end{lemma}

\begin{proof}
For any $1 \leqslant i \leqslant d$,
\begin{equation*}
|W x|^2_i 
= \Big(\sum_{j=1}^d W_{ij} x_j \Big)^2 
= \sum_{j,j'=1}^d W_{ij} W_{ij'} x_j x_{j'}.
\end{equation*}
Thus, by independence,
\begin{equation}    \label{eq:proof-technical-lemma}
\Esp\big(|W x|^2_i\big) 
= \Esp\Big( \sum_{j,j'=1}^d W_{ij} W_{ij'} x_j x_{j'}\Big) 
= \sum_{j=1}^d \Esp(W_{ij}^2) x_j^2 
= s^2 \|x\|^2.
\end{equation}
The result follows by summing over all $i \in \{1, \dots, d\}$.
\end{proof}

We now aim at deriving more involved results on concentration of linear and quadratic forms of sub-Gaussian matrices (Lemma \ref{lemma:bound-deviations-linear} and Lemma \ref{lemma:bound-deviations-quadratic}).
These two propositions are byproducts of the main result of  \citet{kontorovich2014concentration}, which  generalizes McDiarmid's inequality to sub-Gaussian variables. We start by a technical result regarding the sub-Gaussian diameter introduced by \citet{kontorovich2014concentration}, whose definition is recalled below.

\begin{definition} \label{def:subgaussian-diameter}
Let $X$ be a real-valued random variable, $X'$ an independent copy of $X$, and $\varepsilon$ a Rademacher random variable, independent of $X$ and $X'$. Then the sub-Gaussian diameter of $X$ is defined as the smallest $t$ such that $\varepsilon |X-X'|$ is $t^2$ sub-Gaussian.
\end{definition}

\begin{lemma}   \label{lemma:symmetrized-distance}
Let $X$ be an $s^2$ sub-Gaussian centered random variable. Then the sub-Gaussian diameter of $X$ is less than $\sqrt{2}s$.
\end{lemma}

\begin{proof}
Let $\lambda \in \R$. Then, using the notation of Definition \ref{def:subgaussian-diameter}, one has 
\begin{align*}
 \Esp(\exp^{\lambda \varepsilon |X-X'|}) &= \Esp(\exp^{\lambda (X-X')}\mathbf{1}_{\varepsilon=1}) + \Esp(\exp^{\lambda \varepsilon (X'-X)}\mathbf{1}_{\varepsilon=-1}) \\
 &= \Esp(\exp^{\lambda (X-X')}) \\
 &= \Esp(\exp^{\lambda X})^2 \\
 &\leqslant\exp^{2\lambda^2s^2},
\end{align*}
where the last equality is a consequence of the symmetry of $X$.
\end{proof}

We are now ready to prove the two main results of this appendix.

\begin{lemma}[Bound on the deviation of linear forms] \label{lemma:bound-deviations-linear}
Let $V$ be a $\R^{d \times d}$ matrix whose entries are i.i.d~$\nicefrac{s^2}{d}$ sub-Gaussian random variables. Then, for any $x,y \in \R^d$, $x,y \neq 0$,
\[
\Prob\Big(\frac{\langle Vx, y \rangle}{\|x\|\|y\|} \geqslant t\Big) \leqslant2 \exp \Big( - \frac{dt^2}{4s^2} \Big).
\]
\end{lemma}

\begin{proof}
For any $1 \leqslant i, j \leqslant d$, set $X_{ij} = \frac{x_i V_{ij} y_j}{\|x\|\|y\|}$. Let $\mathcal{X} = \R^{d^2}$ endowed with the $\ell_1$ norm, let $X$ be the vector in $\mathcal{X}$ whose $(id+j)$-th coordinate is $X_{ij}$, and let the function $\varphi$ be defined by
$$\varphi: \mathcal{X} \ni Y \longmapsto \sum_{i=1}^{d^2} Y_{i}.$$
By the triangle inequality, $\varphi$ is a Lipschitz continuous function, with Lipschitz constant equal to $1$. Observe also that
$X_{ij}$ is a $\nicefrac{x_i^2s^2y_j^2}{d\|x\|^2\|y\|^2}$ sub-Gaussian. Thus, according to Lemma~\ref{lemma:symmetrized-distance}, the sub-Gaussian diameter of $X_{ij}$ is less than
$\nicefrac{\sqrt{2}x_isy_j}{\sqrt{d}\|x\|\|y\|}$. By \citet[][Theorem 1]{kontorovich2014concentration}, for any $t > 0$, one has
\[
\Prob\left(\varphi(X) \geqslant t\right) \leqslant 2 \exp \Bigg( - \frac{t^2}{2 \sum_{i,j=1}^d \frac{2s^2x_i^2y_j^2}{d\|x\|^2\|y\|^2}} \Bigg),
\]
that is
\[
\Prob\left(\frac{\langle Vx, y \rangle}{\|x\|\|y\|} \geqslant t\right) \leqslant 2 \exp \Big( - \frac{dt^2}{4s^2} \Big).
\]
\end{proof}

\begin{lemma}[Bound of moments of quadratic forms] \label{lemma:bound-deviations-quadratic}
Let $V$ be a $\R^{d \times d}$ matrix whose entries are i.i.d~$\nicefrac{s^2}{d}$ sub-Gaussian random variables, with variance $\nicefrac{1}{d}$. Then, for any $x \in \R^d$, $x \neq 0$,
\[
\Esp\left(\frac{\|Vx\|^{4}}{\|x\|^{4}}\right) \leqslant 1 + \frac{128s^4}{d} \quad \text{ and } \quad  \Esp\left(\frac{\|Vx\|^{8}}{\|x\|^{8}}\right) \leqslant 1 + \frac{384s^4}{d} + \frac{3072s^6}{d^2}.
\]
\end{lemma}

\begin{proof}
The proof is similar to the one of Lemma \ref{lemma:bound-deviations-linear}, with $X_{ij} = \frac{V_{ij} x_j}{\|x\|}$, $\mathcal{X} = \R^d$, and 
$$\varphi_i: \mathcal{X} \ni X \mapsto \sum_{j=1}^d X_{ij}.$$
Each function $\varphi_i$ is a Lipschitz continuous function, with Lipschitz constant equal to $1$. 
Observe now that the random variable $X_{ij}$ is $\nicefrac{x_j^2s^2}{d\|x\|^2}$ sub-Gaussian. Thus, according to Lemma \ref{lemma:symmetrized-distance}, the sub-Gaussian diameter of $X_{ij}$ is less than
$\nicefrac{\sqrt{2}x_js}{\sqrt{d}\|x\|}$. Therefore, according to \citet[][Theorem 1]{kontorovich2014concentration}, for any $t > 0$,
\[
\Prob\left(\varphi_i(X) \geqslant t\right) \leqslant 2 \exp \Bigg( - \frac{t^2}{2 \sum_{j=1}^d \frac{2s^2x_j^2}{d\|x\|^2}} \Bigg),
\]
that is
\begin{equation*}  
\Prob \bigg( \frac{|\langle V_i, x \rangle|}{\|x\|} \geqslant t \bigg) \leqslant 2 \exp \bigg(- \frac{dt^2}{4s^2} \bigg).  
\end{equation*}
Hence \citep[see, e.g.,][]{Pauwels2020},
\begin{equation}    \label{eq:bound-higher-order-moments}
\Esp\bigg(\bigg(\frac{\langle V_i, x \rangle}{\|x\|}\bigg)^{2q}\bigg) \leqslant q! \bigg(\frac{8s^2}{d}\bigg)^q.    
\end{equation}
From identity \eqref{eq:proof-technical-lemma} in the proof of technical Lemma \ref{lemma:variance-random-matrix}, we obtain that, for $q=1$, 
\begin{equation}    \label{eq:equality-norm-two}
\Esp\bigg(\bigg(\frac{\langle V_i, x \rangle}{\|x\|}\bigg)^{2}\bigg) = \frac{1}{d},    
\end{equation}
which is an improvement by a factor $8s^2$ over the previous upper bound. To conclude, it remains to conclude $\|Vx\|^4$ and $\|Vx\|^8$ with the $\langle V_i, x \rangle$. To do so, observe that
\[
\|Vx\|^4 = \bigg( \sum_{i=1}^d \langle V_i, x \rangle^2 \bigg)^2 = \sum_{\substack{i,j=1 \\ i\neq j}}^d \langle V_i, x \rangle^2 \langle V_j, x \rangle^2 + \sum_{i=1}^d \langle V_i, x \rangle^4.
\]
Hence, by independence of the $(V_i)_{1 \leqslant i \leqslant d}$,
\begin{align*}
\Esp \bigg(\frac{\|Vx\|^4}{\|x\|^4}\bigg) 
&= \sum_{\substack{i,j=1 \\ i\neq j}}^d \Esp\bigg(\frac{\langle V_i, x \rangle^2}{\|x\|^2}\bigg) \Esp\bigg(\frac{\langle V_j, x \rangle^2}{\|x\|^2}\bigg) + \sum_{i=1}^d \Esp\bigg(\frac{\langle V_i, x \rangle^4}{\|x\|^4}\bigg) \\
&= d(d-1) \frac{1}{d^2} + d \frac{2 (8s^2)^2}{d^2} \leqslant 1 + \frac{128s^4}{d} \\
&\quad \mbox{(by \eqref{eq:bound-higher-order-moments} and \eqref{eq:equality-norm-two})}
\end{align*}
Similarly,
\[
\|Vx\|^8 = \bigg( \sum_{i=1}^d \langle V_i, x \rangle^2 \bigg)^3 = \sum_{\substack{i,j,k=1 \\ i\neq j \neq k}}^d \langle V_i, x \rangle^2 \langle V_j, x \rangle^2 \langle V_j, x \rangle^2 + \sum_{\substack{i,j=1 \\ i\neq j}}^d \langle V_i, x \rangle^2 \langle V_j, x \rangle^4 + \sum_{i=1}^d \langle V_i, x \rangle^8.
\]
Hence,
\begin{align*}
\Esp \bigg(\frac{\|Vx\|^8}{\|x\|^8}\bigg) 
&= \sum_{\substack{i,j,k=1 \\ i\neq j \neq k}}^d \Esp\bigg(\frac{\langle V_i, x \rangle^2}{\|x\|^2}\bigg) \Esp\bigg(\frac{\langle V_j, x \rangle^2}{\|x\|^2}\bigg) \Esp\bigg(\frac{\langle V_k, x \rangle^2}{\|x\|^2}\bigg) \\
& \qquad + \sum_{\substack{i,j=1 \\ i\neq j}}^d \Esp\bigg(\frac{\langle V_i, x \rangle^4}{\|x\|^4}\bigg) \Esp\bigg(\frac{\langle V_j, x \rangle^2}{\|x\|^2}\bigg) + \sum_{i=1}^d \Esp\bigg(\frac{\langle V_i, x \rangle^8}{\|x\|^8}\bigg) \\
&= d(d-1)(d-2) \frac{1}{d^2} + 3d(d-1) \frac{2 (8s^2)^2}{d^3} + d \frac{6 (8s^2)^3}{d^3} \\
&\leqslant 1 + \frac{384s^4}{d} + \frac{3072s^6}{d^2}.
\end{align*}
\end{proof}

\section{A Version of the Picard-Lindel\"{o}f Theorem}
\label{sec:picard-lindelof}

\begin{theorem}
\label{thm:picard-lindelof}
Assume that $f$ is Lipschitz continuous in its first argument and continuous in its second one. Then, for any $z \in \R^d$, the initial value problem
\begin{equation}\label{eq:initial_value_problem}
    dH_t = f(H_t, t)dt, \quad H_0 = z,
\end{equation}
admits a unique solution $H:[0,1] \to \R^d$.
\end{theorem}
\begin{proof}
    Let $ \mathscr{C}([s,t], \R^{d})$ be the set of continuous functions from $[s,t]$ to $\R^{d}$. For any $[s,t] \subset [0,1]$, $\zeta \in \R^{d}$, let  $\Psi$ be the function
    \begin{align*}
        \Psi: \mathscr{C}([s,t], \R^{d}) &\to \mathscr{C}([s,t], \R^{d}) \\
        Y &\mapsto \big(v \mapsto \zeta + \int_{s}^v f(Y_u, u) du \big).
    \end{align*}
For any $Y, Y' \in \mathscr{C}([s,t], \R^{d})$, $v \in [s,t]$, one has, denoting by $K_f$ the Lipschitz constant of $f$ in its first argument,
\begin{align*}
    \| \Psi(Y)_v - \Psi(Y')_v\| &\leqslant \int_{s}^v \big\| \big(f(Y_u, u) - f(Y'_u, u) \big)du \big\| \\
    &\leqslant \int_{s}^v K_{f} \|Y_u - Y'_u\|du \\
    & \leqslant K_{f} \int_{s}^v \|Y - Y'\|_{\infty}du \\
    & \leqslant K_{f} \| Y-Y'\|_\infty (t-s).
\end{align*}
This yields
\[\| \Psi(Y) - \Psi(Y')\|_{\infty} \leqslant K_f (t-s) \| Y-Y'\|_\infty, \]
which means that the function $\Psi$ is Lipschitz continuous on $\mathscr{C}([s,t], \R^{d})$ endowed with the supremum norm, with Lipschitz constant $K_{f}(t-s)$. So, on any interval $[s,t]$ of length smaller than $\delta = \nicefrac{1}{2} K_f$, the function $\Psi$ is a contraction. Thus, by the Banach fixed-point theorem, for any initial value $\zeta$, $\Psi$ has a unique fixed point. Hence, there exists a unique solution to \eqref{eq:initial_value_problem} on any interval of length $\delta$ with any initial condition. To obtain a solution on $[0,1]$ it is sufficient to concatenate these solutions.
\end{proof}

\section{Detailed Experimental Setting and Additional Plots}
\label{apx:experimental-setting}

Our code is available at \url{https://github.com/PierreMarion23/scaling-resnets}.

To obtain Figures \ref{fig:scaling_init_norm} to \ref{fig:scaling_init_gradient}, we initialize ResNets from \texttt{res-3} with the hyperparameters of Table \ref{tab:hyperparams-1}.
\begin{table}[ht]
    \centering
    \begin{tabular}{cc}
    \toprule
    {\bf Name} & {\bf Value} \\
    \midrule
    $d$ & $40$ \\
    $n_\textnormal{in}$ & $64$ \\
    $n_\textnormal{out}$ & $1$ \\
    $L$ & $10$ to $1000$ \\
    $\beta$ & $0.25, 0.5, 1$ \\
    weight distribution & $\mathcal{U}(-\sqrt{3/d}, \sqrt{3/d})$ \\
    data distribution & standard Gaussian \\
    \bottomrule
    \end{tabular}
    \caption{Hyperparameters of Figures \ref{fig:scaling_init_norm} to \ref{fig:scaling_init_gradient}}
    \label{tab:hyperparams-1}
\end{table}

\begin{figure}[!p]
    \centering
    \includegraphics[width=\textwidth]{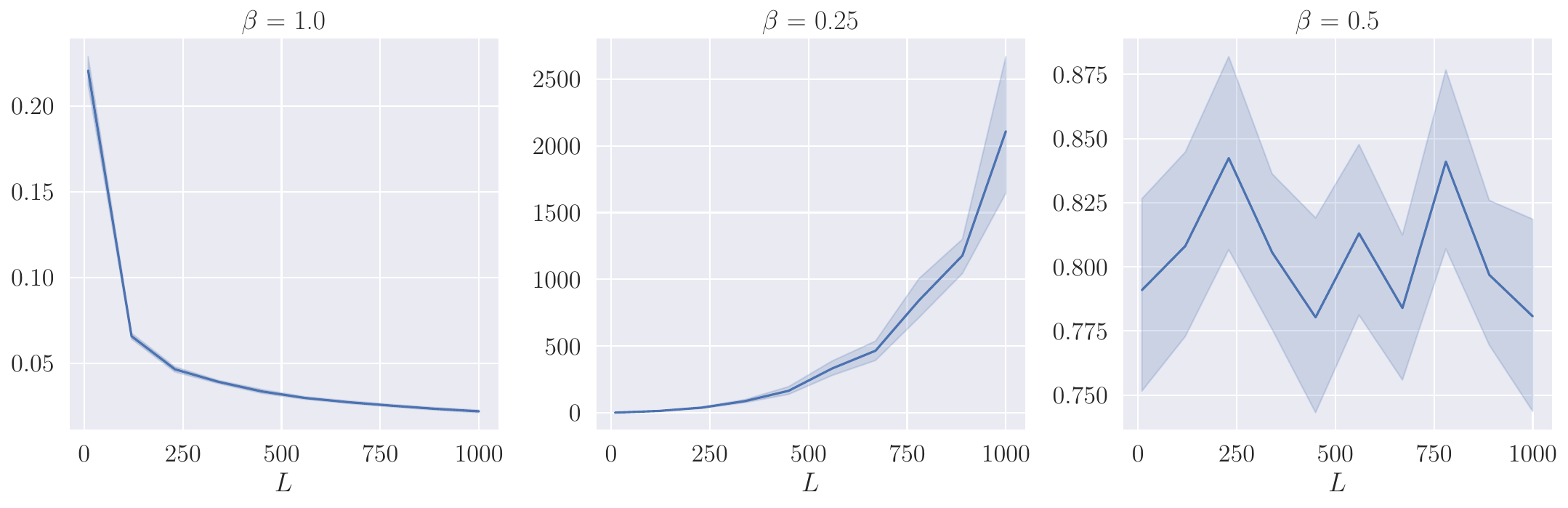}
    \caption{Evolution of $\|h_L - h_0\|/\|h_0\|$ as a function of $L$ for different values of~$\beta$ and for the model $h_{k+1} = h_k + \alpha_L \sigma(W_k h_k)$. Hyperparameters are as in Figure \ref{fig:scaling_init_norm}.}
    \label{fig:scaling_init_norm_inner_only}
\end{figure}

\begin{figure}[!p]
    \centering
    \begin{subfigure}[b]{0.49\textwidth}    
        \centering
        \includegraphics[width=\textwidth]{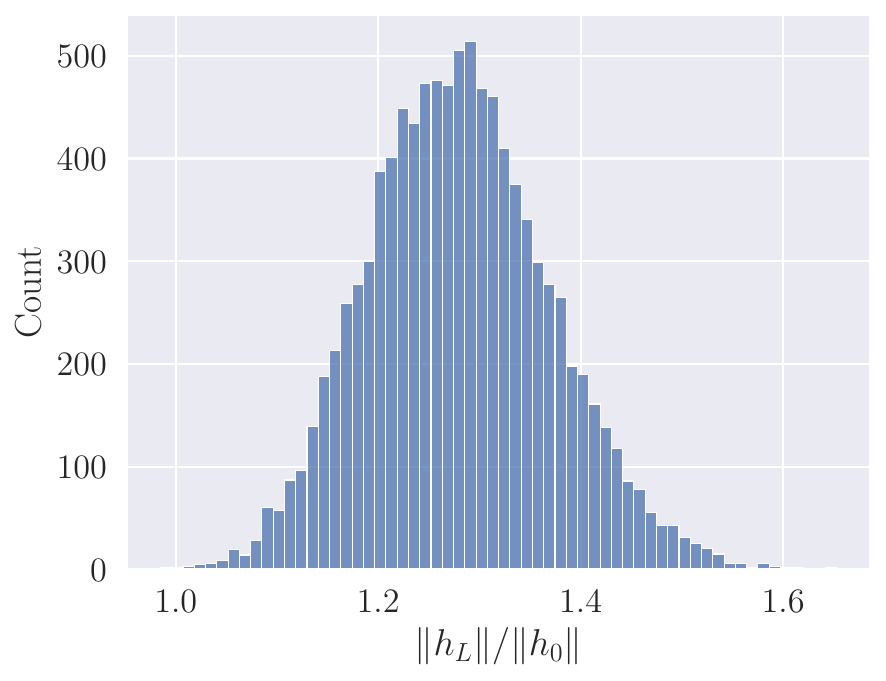}
        \caption{Distribution of $\|h_L\|/\|h_0\|$}
        \label{fig:distribution-forward_inner_only}
    \end{subfigure}
    \hfill
    \begin{subfigure}[b]{0.49\textwidth}
        \centering
        \includegraphics[width=\textwidth]{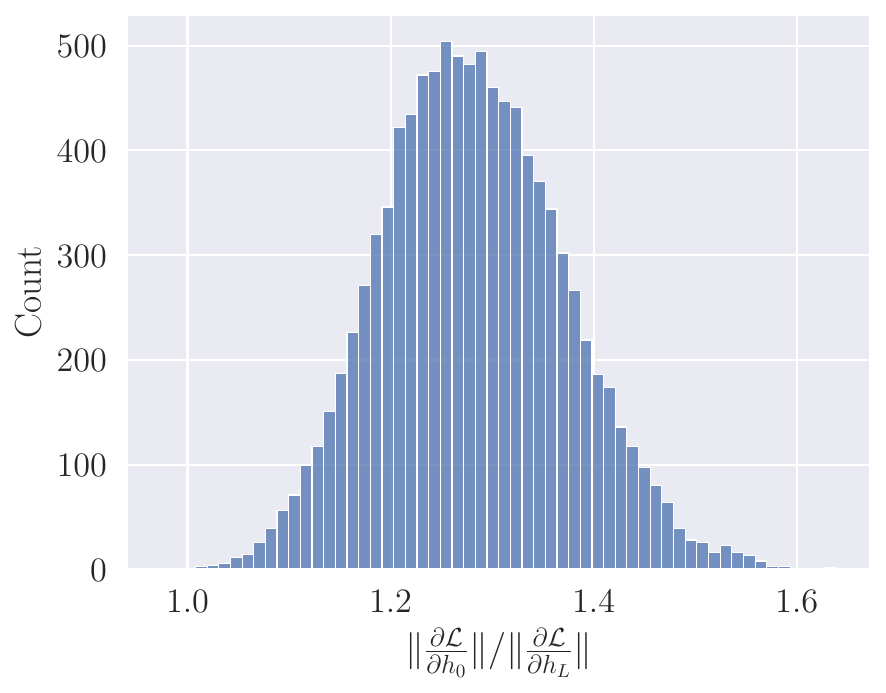}
        \caption{Distribution of $\|\frac{\partial \mathscr{L}}{\partial h_0}\|/ \|\frac{\partial \mathscr{L}}{\partial h_L}\|$}
        \label{fig:distribution-gradients_inner_only}
    \end{subfigure}
   \caption{Empirical distributions of the norms for the model $h_{k+1} = h_k + \alpha_L \sigma(W_k h_k)$. Hyperparameters are as in Figure~\ref{fig:distributions}.}
\end{figure}

\begin{figure}[!p]
    \centering
    \includegraphics[width=\textwidth]{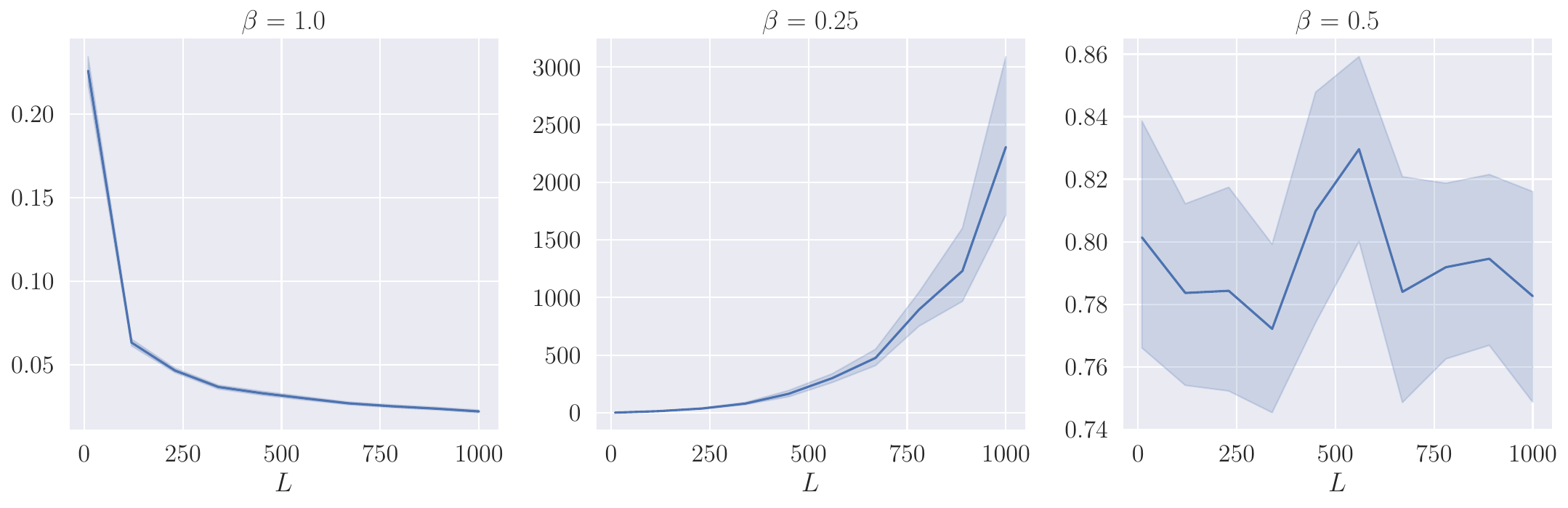}
    \caption{Evolution of $\|p_0 - p_L\|/\|p_L\|$ as a function of $L$ for different values of~$\beta$ and for the model $h_{k+1} = h_k + \alpha_L \sigma(W_k h_k)$. Hyperparameters are as in Figure~\ref{fig:scaling_init_gradient}.}
    \label{fig:scaling_init_gradient_inner_only}
\end{figure}

\begin{figure}[!p]
    \centering
    \includegraphics[width=\textwidth]{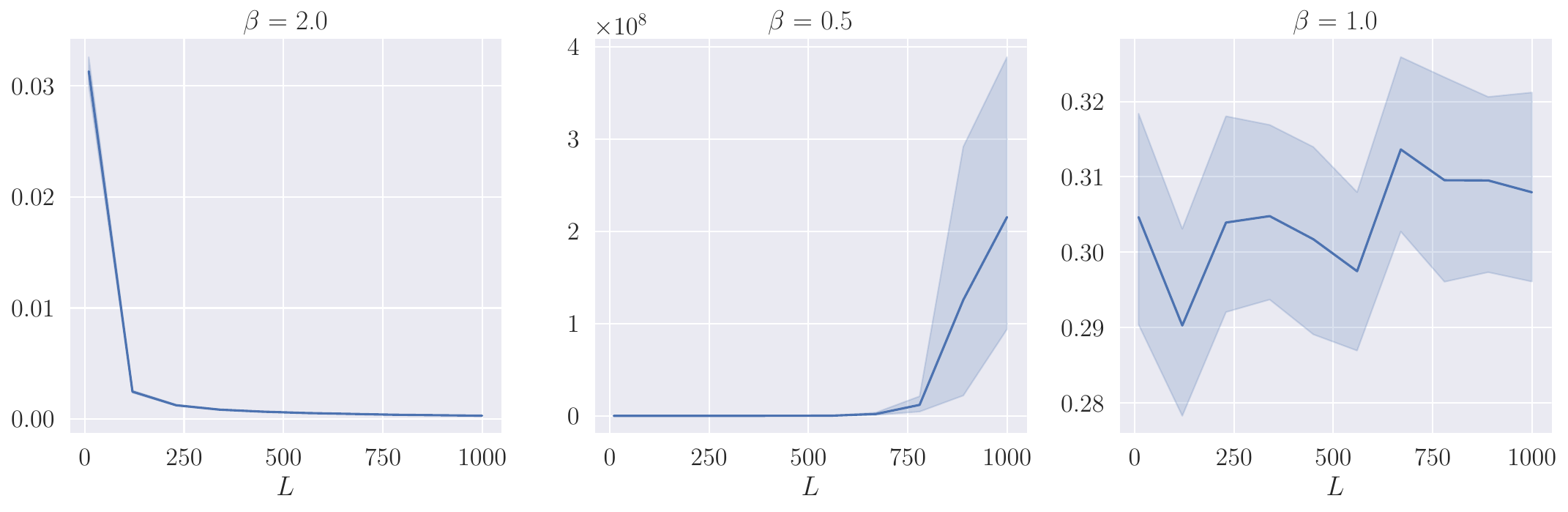}
    \caption{Evolution of $\|h_L-h_0\|/\|h_0\|$ as a function of $L$ for different values of~$\beta$ and a smooth initialization of the model $h_{k+1} = h_k + \alpha_L \sigma(W_k h_k)$. Hyperparameters are as in Figure~\ref{fig:smooth_scaling_init_norm}.}
    \label{fig:smooth_scaling_init_norm_inner_only}
\end{figure}

\begin{figure}[!p]
    \centering
    \includegraphics[width=\textwidth]{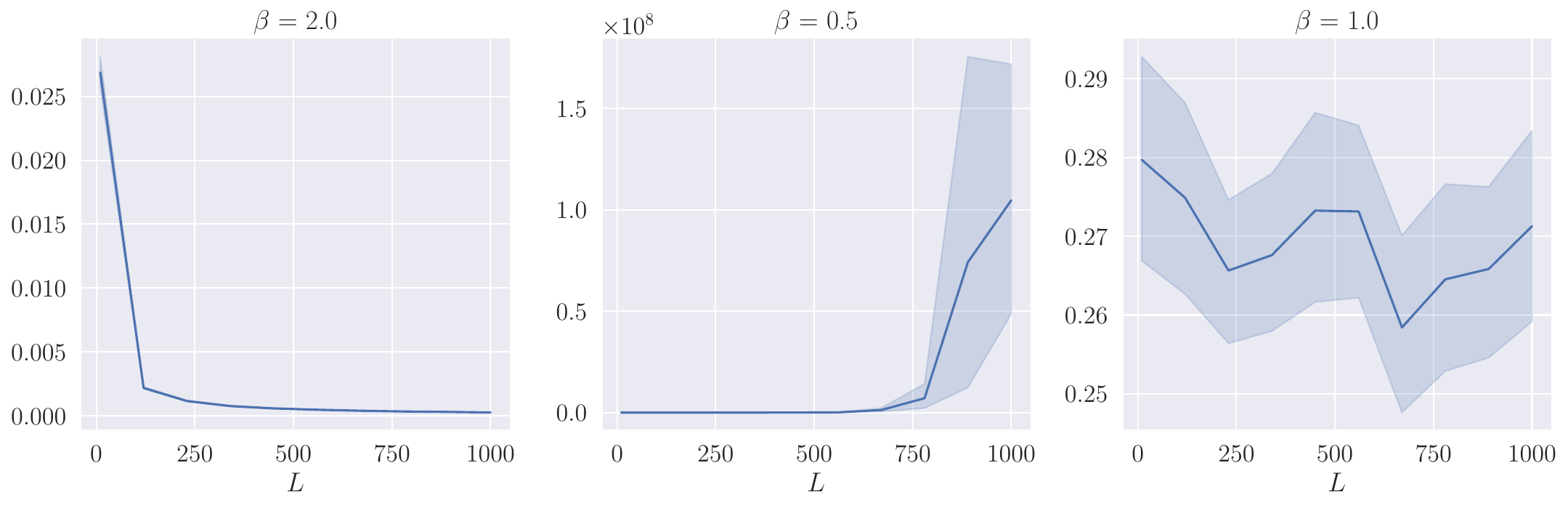}
    \caption{Evolution of $\|p_0-p_L\|/\|p_L\|$ as a function of $L$ for different values of~$\beta$ and a smooth initialization of the model $h_{k+1} = h_k + \alpha_L \sigma(W_k h_k)$. Hyperparameters are the same as in Figure \ref{fig:smooth_scaling_init_gradient}.}
    \label{fig:smooth_scaling_init_gradient_inner_only}
\end{figure}

Each experiment is repeated $50$ times, with independent data and weight sampling. 

For Figures \ref{fig:smooth_scaling_init_norm} and \ref{fig:smooth_scaling_init_gradient}, we take the same hyperparameters except for $\beta$, which now takes values in $\{0.5, 1, 2\}$, and for the weight distribution. The weights are now initialized as discretizations of a Gaussian process. More precisely, each entry of $\mathscr{V}$ and $\mathscr{W}$ is an independent Gaussian process with zero mean and an RBF kernel of variance $10^{-2}$.

We also perform the same experiments with the model $h_{k+1} = h_k + \alpha_L \sigma(W_k h_k)$, and report the result in Figures \ref{fig:scaling_init_norm_inner_only}--\ref{fig:smooth_scaling_init_gradient_inner_only}. Although this formulation is not covered by our theoretical results (see discussion at the end of Section \ref{subsec:model}), we observe qualitatively similar results as Figures \ref{fig:scaling_init_norm}--\ref{fig:smooth_scaling_init_gradient}.

To obtain Figure \ref{fig:heatmap}, we take the hyperparameters of Table \ref{tab:hyperparams-2}.

\begin{table}[ht]
\centering
\begin{tabular}{cc}
    \toprule
    {\bf Name} & {\bf Value} \\
    \midrule
    $d$ & $40$ \\
    $n_\textnormal{in}$ & $64$ \\
    $n_\textnormal{out}$ & $1$ \\
    $L$ & $1000$ \\
    $\beta$ & $0.2$ to $1.3$ \\
    weight distribution & fractional Brownian motion \\
     & with Hurst index from $0.05$ to $0.97$ \\
    data distribution & standard Gaussian \\
    \bottomrule
\end{tabular}
    \caption{Hyperparameters of Figure \ref{fig:heatmap}}
    \label{tab:hyperparams-2}
\end{table}

More precisely, for each $1 \leqslant i,j \leqslant d$, we let $(V_{k+1, i, j})_{0 \leqslant k \leqslant L-1}$ be the increments of a fractional Brownian motion (fBm), where the various fBm involved are independent. The procedure is the same for $w$.

In Figure \ref{fig:heatmap-trained}, we use \texttt{res-1}, with the hyperparameters of Table \ref{tab:hyperparams-3}. Each residual layer also includes a bias term, initialized to zero, and trained with the same learning rate as the other weights. We train on MNIST\footnote{\url{http://yann.lecun.com/exdb/mnist}} and CIFAR-10\footnote{\url{https://www.cs.toronto.edu/~kriz/cifar.html}} using the Adam optimizer \citep{kingmaAdamMethodStochastic2017} for $10$ epochs. The learning rate is divided by $10$ after $5$ epochs. The best performance on the learning rate grid is reported in the figure. 
\begin{table}[ht]
    \centering
    \begin{tabular}{cc}
    \toprule
    {\bf Name} & {\bf Value} \\
    \midrule
    $d$ & $30$ \\
    $L$ & $1000$ \\
    $\beta$ & $0.2$ to $1.3$ \\
    weight distribution & fractional Brownian motion \\
     & with Hurst index from $0.05$ to $0.97$ \\
    learning rate grid & $10^{-4}, 10^{-3}, 10^{-2}, 10^{-1}, 1$ \\
    \bottomrule
    \end{tabular}
    \caption{Hyperparameters of Figure \ref{fig:heatmap-trained}}
    \label{tab:hyperparams-3}
\end{table}

\begin{figure}[!p]
     \centering
     \begin{subfigure}[b]{0.32\textwidth}
         \centering
         \includegraphics[width=\textwidth]{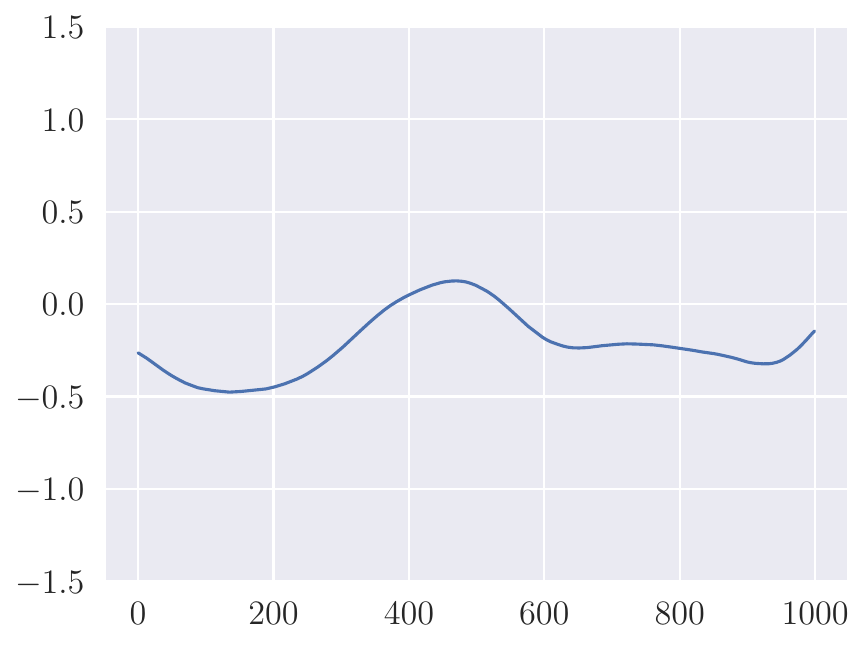}
         \caption{$\beta=1$, smooth initialization}
     \end{subfigure}
     \hfill
     \begin{subfigure}[b]{0.32\textwidth}
         \centering
         \includegraphics[width=\textwidth]{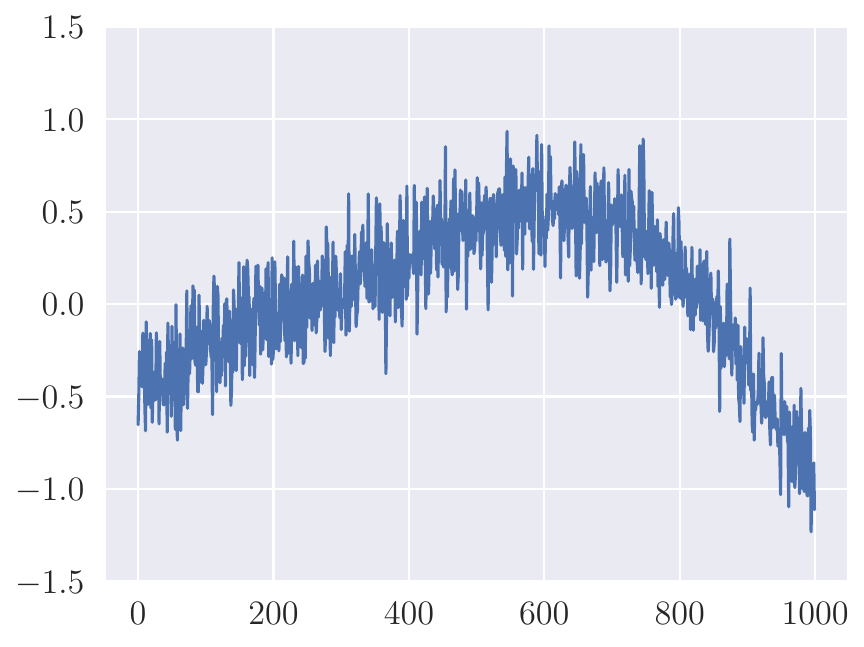}
         \caption{$\beta=1$, i.i.d.~initialization}
     \end{subfigure}
     \hfill
     \begin{subfigure}[b]{0.32\textwidth}
         \centering
         \includegraphics[width=\textwidth]{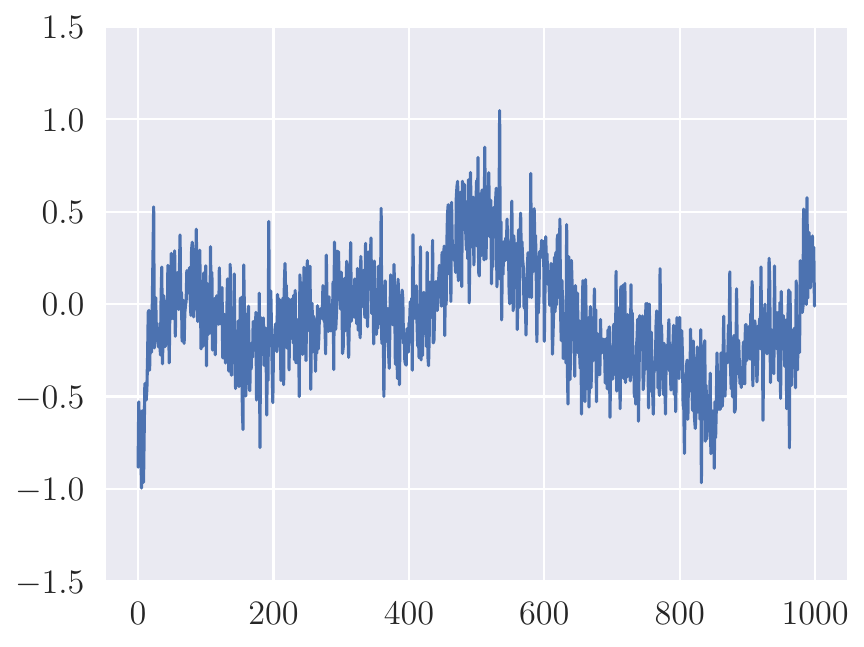}
         \caption{$\beta=\nicefrac{1}{2}$, i.i.d.~initialization}
     \end{subfigure}
     \caption{Plot of a given coordinate of $w_k$, after training, as a function of the layer index $k$ ranging from $1$ to the depth $L=1000$ for three different choices of $\beta$ and initializations. A bias term is trained in each residual layer.}
     \label{fig:ex-weights-after-training-with-bias}
\end{figure}

Figure \ref{fig:ex-weights-after-training} is obtained by plotting a random coordinate of $w_k$, after training on MNIST. While there is no bias term in the model trained for this figure, we show below the result of the same experiment when additionally training a zero-initialized bias term in each residual layer. We observe that adding a bias term does not qualitatively change the results.

\bibliography{sample}

\end{document}